\documentclass[journal,doublecolumn]{IEEEtran}
\usepackage[T1]{fontenc}

\usepackage{amsfonts}
\usepackage{amsmath}
\usepackage{booktabs}
\usepackage{mathrsfs}  
\usepackage{hyperref}
\usepackage{float}
\usepackage{xcolor}
\usepackage{threeparttable}
\usepackage{amssymb}
\usepackage{stfloats}
\usepackage{graphicx}
\usepackage{placeins}
\usepackage{amsthm}
\usepackage{subfigure}
\usepackage{array}
\usepackage{nomencl}
\makenomenclature 
\usepackage[ruled,vlined]{algorithm2e}
\usepackage{bm}
\usepackage[square, comma, sort&compress, numbers]{natbib}
\usepackage{fancynum}
\definecolor{mygray}{gray}{.9}
\usepackage{colortbl}
\newtheorem{theorem}{Theorem}
\newtheorem{lemma}{Lemma}

\newtheorem{defn}{Definition}
\newtheorem{remark}{Remark}
\newtheorem{assum}{Assumption}

\begin{document}
%
\title{ Online Orthogonal Dictionary Learning Based on Frank-Wolfe Method}
\author{Ye~Xue,~\IEEEmembership{Graduate Student Member,~IEEE,}
	and Vincent~LAU,~\IEEEmembership{Fellow,~IEEE}
	\thanks{Y. Xue and V. Lau are with the Department of Electronic and Computer Engineering, Hong Kong University of Science and Technology, Hong Kong (E-mail: yxueaf,  eeknlau@ust.hk).}}%

\markboth{Journal of \LaTeX\ Class Files,~Vol.~14, No.~8, August~2015}%
{Shell \MakeLowercase{\textit{et al.}}: Bare Demo of IEEEtran.cls for IEEE Journals}



\maketitle

\begin{abstract}
Dictionary learning is a widely used unsupervised learning method in signal processing and machine learning. Most existing works on  dictionary learning adopt an offline approach, and there are  two main offline ways of conducting it. One is to  alternately optimize both the dictionary and the sparse code, while the other is to optimize the dictionary by restricting it over the orthogonal group. The latter,   called  orthogonal dictionary learning,  has a lower implementation complexity, and hence,  is more favorable for low-cost devices. However, existing schemes for orthogonal dictionary learning only work with  batch data and cannot be implemented online, making them inapplicable  for real-time applications. This paper thus proposes a novel online orthogonal  dictionary  scheme to dynamically learn the dictionary from streaming data, without storing the historical data. The proposed scheme includes a novel  problem formulation and an efficient online algorithm design with convergence analysis. In the problem formulation, we  relax  the orthogonal constraint to enable an efficient online algorithm.  We then propose the design of a new Frank-Wolfe-based online algorithm with a convergence rate of $\mathcal{O}(\ln t/t^{1/4})$. The convergence rate in terms of key system parameters is also derived. Experiments with synthetic data and real-world internet of things (IoT) sensor readings demonstrate the effectiveness and  efficiency of the proposed online orthogonal dictionary learning scheme.
\end{abstract}


%
\IEEEpeerreviewmaketitle

\section{Introduction}
\IEEEPARstart{S}{parse} representation  of data has been widely used in signal processing, machine learning and data analysis, and shows a highly expressive and effective representation ability \cite{Aharon2006,mairal2007sparse,energydis,spatial}. It represents the data $\bm{y}\in \mathbb{R}^N$  by a  linear combination  $\bm{y}\approx\bm{D}^{true}\bm{x}$, where  $\bm{x}\in\mathbb{R}^M$ is a sparse code (i.e., the number of non-zero
entries of $\bm{x}$ is much smaller than $M$), and $\bm{D}^{true}\in\mathbb{R}^{N\times M}$ is the dictionary  that contains the compact information of $\bm{y}$. At the initial stage in this line of research, predefined dictionaries
based on a Fourier basis  and various types of wavelets were successfully used for signal processing. However, using a learned dictionary instead of a generic one has been shown to dramatically improve the performance on various tasks, e.g., image denoising  and classification.

One method to learn a dictionary is to  alternately optimize (AO) problems with both the dictionary and the sparse code as the variables {\cite{Aharon2006,mairal2007sparse,energydis,spatial,khodayar2021deep}}. In this approach,  the dictionary usually has no constraints {  or has a bounded norm constraint on each atom\footnote{One column of the dictionary matrix $\bm{D}$ is called an atom.} of the dictionary to prevent trivial solutions  \cite{energydis,spatial}.}  Another method is to restrict  the dictionary over the orthogonal group $\mathbb{O}(N,\mathbb{R})$ and solve the following optimization problem only for the dictionary \cite{sun2015complete1,bai2018subgradient, zhai2020complete}:
\begin{equation}\label{eq:DLo}  
\begin{aligned}
\underset{\bm{D}\in \mathbb{O}(N,\mathbb{R})}{\text{minimize}}\quad
& \frac{1}{T} \sum^T_{t=1}{\text Sp}(\bm{D}^{\text T}\bm{y}_t),
\\
\end{aligned}
\end{equation} 
where  ${\text Sp}(\cdot)$ is a sparsity-promoting function. This formulation is motivated by the fact that  the sparse code can be obtained as $\{\bm{x}_t=({\bm{D}}^{true})^T\bm{D}^{true}\bm{x}_t\approx({\bm{D}}^{true})^{\text T}\bm{y}_t\}^T_{t=1}$  if the dictionary $\bm{D}^{true}$ is orthogonal. 

 To  enable efficient data processing,  the latter method, called  orthogonal dictionary learning (ODL), is more favorable than the AO  method for the following reasons. First,  ODL  has a lower computational complexity. It updates only the dictionary at each iteration, while the AO method requires  solving two sub-problems respectively for the dictionary and the  sparse code. {  Second, ODL has a lower sample complexity since it restricts the dictionary over a smaller optimization space.\footnote{Though the orthogonal constraint seemingly brings performance loss as it narrows the optimization space,  ODL has a  performance competitive with that of the AO method in real applications \cite{bao2013fast}.}   Third, ODL allows efficient  transmission of the	dictionary. This is because an $N\times N$ orthogonal matrix can be  represented by  $\frac{N(N-1)}{2}$ statistically independent angles via the Givens rotation,\footnote{Although matrices belonging to orthogonal group $\mathbb{O}(N)$, but not the special orthogonal group $\mathbb{SO}(N)$, cannot be directly factorized by the Givens rotations, a factorization can be obtained up to a permutation with a negative sign, e.g., by flipping two columns \cite{frerix2019approximating}.} and the angles can be quantized efficiently to achieve a minimum quantization loss \cite{yuen2012beamforming}. This angle-based transmission has been adopted in many wireless communication standards \cite{gast2013802}.}

Despite the benefits of ODL, however, the existing ODL methods  only work with batch data. In other words, they require  the whole data set to run the algorithm. Hence, they are not applicable in many  real-time applications,  including real-time network monitoring \cite{babu2001continuous},
sensor networks \cite{de2016iot},   and Twitter analysis \cite{bifet2010sentiment}, where data arrives continuously in rapid,  unpredictable, and unbounded streams. To deal with  streaming data, we propose an  online ODL approach that processes the data in a single sample or  in a mini-batch. The  online ODL problem can be regarded as  taking  ${ T}\to \infty$ in Problem (\ref{eq:DLo}), and  is formally formulated as
 \begin{equation}\label{eq:DLonline}  
\begin{aligned}
\underset{\bm{D}\in\mathbb{O}(N,\mathbb{R})}{\text{minimize}}\quad
& F(\bm{D})=\mathbb{E}_{\bm{y}\sim P}[{\text Sp}(\bm{D}^{\text T}\bm{y})],
\\
\end{aligned}
\end{equation}
 where $\bm{y}$ is the realization of the random variable $Y$ drawn from a distribution $P$. Since the distribution $P$ is usually unknown and the cost of computing the expectation is prohibitive, the main challenge is to  solve Problem (\ref{eq:DLonline}) without the accessibility of  $F(\bm{D})$ or its gradient $\nabla F(\bm{D})$. In this case,  we can only rely on the sampled data to calculate the approximation of the objective function or the gradient. These approximations will jeopardize the performance and the convergence of the algorithms compared to the case where the exact objective value and gradient are available \cite{johnson2013accelerating}. 
 
 {  Problem (\ref{eq:DLonline}) is an online constraint optimization problem and many general  algorithms are available  for solving  this type of problem}, such as regularized dual averaging (RDA) \cite{xiao2010dual} and stochastic mirror descent (SMD) \cite{zhou2017stochastic}.  However, none  can be directly applied to Problem (\ref{eq:DLonline}) due to the nonconvex constraint set and possibly nonconvex sparsity-promoting function, e.g., $\text{Sp}(\cdot)=-\|\cdot\|_p^p,(p\in\mathbb{N},\quad p>2)$ \cite{zhai2020complete,shen2020complete,blind}. { In this work, to enable an efficient  online algorithm with a convergence guarantee}, we propose a novel online ODL solution with a convex relaxation of the  orthogonal constraint in Problem (\ref{eq:DLonline}).   After the relaxation, the problem becomes a nonconvex optimization over a convex set.  {  One could  transform  this relaxed problem  into an unconstrained problem  with a composite nonconvex objective $\bar{F}(\bm{D})$ (see Example 3 in \cite[Section 1.1]{xiao2010dual}) and solve the transformed problem   by ProxSGD \cite{ghadimi2016mini}.  However,  the use of ProxSGD gives rise to o  issues: first, the proximal operator  in  ProxSGD creates a high per-iteration computational complexity; and second, ProxSGD  can only be guaranteed to converge to an $\epsilon$-stationary point (a point $\bm{D}^*$, such that $\mathbb{E}[\|\nabla \bar{F}(\bm{D}^*)\|^2_F]\leq \epsilon$)} when  the mini-batch size is increasing by $1/\epsilon$ \cite{ghadimi2016mini}. To achieve small error, the mini-batch size needs to be large, which is not suitable  for most online processors with limited memory.  

In this work, we propose a Frank-Wolfe-based \cite{frank1956algorithm} algorithm,  the Nonconvex Stochastic Frank-Wolfe (NoncvxSFW) method, to solve the relaxed problem directly without the problem transformation. NoncvxSFW  can achieve low-complexity per-iteration computation thanks to the  linear minimization oracle (LMO) in the Frank-Wolfe method.  We also prove that the proposed algorithm with a single sample or a  fixed mini-batch size can be guaranteed to converge to a stationary point of the relaxed online ODL problem.   The main contributions are summarized as follows.
\begin{itemize}
	\item {\bf Novel Online ODL Formulation}: We propose an  online ODL problem with an $\ell_3$-norm-based sparsity-promoting function and a convex relaxation of the orthogonal constraint. We prove that all the optimal solutions of the original problem are also the optimal solutions of the relaxed problem, which enables an efficient  online algorithm with guaranteed  convergence. 
	\item {\bf Online Frank-Wolfe-Based Algorithm}: We develop an online  algorithm, the NoncvxSFW method, to solve the relaxed optimization problem with a single sample or a fixed mini-batch size. The convergence is analyzed, and  its rate is  shown to be  $\mathcal{O}(\ln t/t^{1/4})$, where $t$ is the number of iterations.  As far as we are aware, this is the first non-asymptotic  convergence rate for  online nonconvex   optimization  using the Frank-Wolfe-based method.  The proposed algorithm and the corresponding theoretical results can also be generalized into  general online nonconvex problems with convex constraints. 
	\item {\bf Effective and Efficient Application on IoT Sensor Data Compression}: We provide extensive simulations with both synthetic data and a real-world IoT sensor data set. The simulation results demonstrate
	the effectiveness and efficiency of our proposed online ODL method. They also verify the correctness of our theoretical results. For the synthetic data, the proposed scheme can achieve superb
	performance in terms of the convergence rate and the recovery error, while for the real-world sensor data, it can achieve a better root-mean-square error (RMSE) with a higher compression ratio   for sensor data compression compared to the state-of-the-art baselines \cite{zhai2020complete,mokhtari2020stochastic,mairal2009online,spatial,akhtar2017nonparametric}.
\end{itemize}

   { The rest of the paper is organized as follows. In Section \ref{sec:sys} we illustrate the  signal model and present the  problem formulation for the online ODL. In Section \ref{sec:alg}, we present the Frank-Wolfe-based online algorithm with non-asymptotic convergence analysis. Application examples and numerical simulation results are provided in Section \ref{sec:app} and Section \ref{sec:exp}, respectively. Finally, Section \ref{sec:concl} summarizes the work.}
{  
\nomenclature[01]{\(\bm{x},\bm{X}\)}{Column vector and matrix}
\nomenclature[02]{\(\bm{X}_{n,:},\bm{X}_{:,i}\)}{The $n$-th row and the $i$-th column of matrix $\bm{X}$}
\nomenclature[03]{\(x_{i},X_{i,j}\)}{The $i$-th element of vector $\bm{x}$ and the element in the $i$-th row and the $j$-th column of matrix $\bm{X}$ }
\nomenclature[04]{\(\bm{e}_i\)}{Standard basis vector with a $1$ in the $i$-th coordinates and $0$ elsewhere}
\nomenclature[05]{\(\mathbb{O}(N,\mathbb{R})\)}{ $N$-dimensional orthogonal group with real-valued entries}
\nomenclature[06]{\(\mathbb{B}_{sp}(N,\mathbb{R})\)}{$N$-dimensional (closed) unit spectral ball}

\nomenclature[07]{\((\centerdot)^{\text T}, vec(\centerdot) \)}{Transpose and vectorization}
\nomenclature[08]{\(diag(\bm{x})\)}{Diagonal matrix  with  vector $\bm{x}$ on  its diagonal }
\nomenclature[09]{\(\lceil\cdot \rceil, \lfloor\cdot\rfloor, \lvert\cdot\rvert\)}{ Element-wise ceiling operator, element-wise floor operator and taking element-wise absolute
	value  }
\nomenclature[10]{\([a]\_b\)}{$a$ modulo $b$}
\nomenclature[11]{\([N]\)}{Set $\{1, 2,...,N\}$}
\nomenclature[12]{\(\{\bm{x}_i\}_{i=1}^N\)}{Set $\{\bm{x}_i; i\in [N]\}$}
\nomenclature[13]{\(Tr(\cdot)\)}{Trace}
\nomenclature[14]{\(\langle{\bm{X}},\bm{Y}\rangle\)}{General
	inner product of $\bm{X}$ and $\bm{Y}$}
\nomenclature[15]{\(\lVert\,{\cdot}\,\rVert_p \)}{$\ell_{p}$-norm of a vector or the induced   $\ell_{p}$-norm  of a matrix}
\nomenclature[16]{\(\lVert\,{\cdot}\,\rVert_F,\lVert\,{\cdot}\,\rVert_*\)}{Matrix  Frobenius norm and nuclear norm}
\nomenclature[17]{\(\lVert\,{\cdot}\,\rVert\)}{$\ell_2$-norm of a vector or the spectral norm of a matrix}
\nomenclature[18]{\(\odot,(\cdot)^{\odot p}\)}{Hadamard product and the element-wise $p$-th power }
\nomenclature[19]{\(\mathcal{P}_{\mathbb{O}(N,\mathbb{R})}(\bm{D})\)}{Projects $\bm{D}$ onto the orthogonal group }
\nomenclature[20]{\(Polar(\bm{D})\)}{Polar decomposition of matrix $\bm{D}$}
\nomenclature[21]{\(\mathbb{E}[\centerdot],\mathbb{E}_{t}[\centerdot]\)}{Expectation over all the randomness in the system, expectation conditioned on the randomness up until  time $t$}
\nomenclature[22]{\(\mathbb{E}_{\bm{x}}[\centerdot],\mathbb{E}_{\bm{x}\sim P}[\centerdot] \)}{Expectation over random variable $\bm{x}$, expectation over random variable $\bm{x}$ with distribution $P$}
\nomenclature[23]{\(\bm{x} \overset{i.i.d}{\sim} \mathcal{BG}(\theta)\)}{Vector $\bm{x}$ has i.i.d elements and each is a product of independent Bernoulli and standard normal random variables: $x_i =
	b_ig_i$, where $b_i \sim Ber(\theta)$ and $g_i  \sim \mathcal{N} (0, 1)$}
\renewcommand{\nomname}{Notations}
\printnomenclature[0.8in]
 }

%
%
%

\section{ Signal Model and Problem Formulation}\label{sec:sys}
In this section, we  introduce   the signal model and the proposed online ODL problem formulation.

\subsection{Signal Model}
In the online ODL, we consider that the data samples  arrive in streams. At time $t$, there is one mini-batch of samples  $\bm{Y}_t =[\bm{y}^1_t,\ldots,\bm{y}_t^{M_t}]\in\mathbb{R}^{N\times {M_t}}$ arriving at the processor, where $M_t$ is the mini-batch size at time $t$ and $\bm{y}^j_t$ is the $j$-th sample in the $t$-th mini-batch.  Each sample is assumed to be generated by
\begin{equation}\label{mod}
\bm{y}^j_t =\bm{D}^{true}\bm{x}^j_t, \forall j=1,\ldots,M_t,  \forall t,
\end{equation} 
where $\bm{D}^{true}$ is the orthogonal dictionary and $\bm{x}^j_t$ is a realization of the random variable $X$ drawn from some distribution that induces sparsity. The basic goal of the online ODL processor is to dynamically learn the dictionary  in an online manner without storing all the historical samples. That is to say, the processor updates  the dictionary $\bm{D}_t$ at time $t$ only according to the samples arriving at that time, i.e., $\bm{Y}_t$. 

 \subsection{$\ell_3$-Norm-Based Formulation}
For the online ODL scheme, the dynamic updating of the dictionary can be done by solving the generic Problem  (\ref{eq:DLonline}) in an online manner. In Problem (\ref{eq:DLonline}), the  sparsity-promoting function $Sp(\cdot)$ needs to be carefully designed since it determines the performance of the online ODL and the complexity of the algorithm. In this work, we use $Sp(\cdot)=-\Vert\cdot\Vert^3_3$, which results in the following optimization problem:
 \begin{equation}
 \label{eq:DLonlinel31}  
\begin{aligned}
\underset{\bm{D}\in \mathbb{O}(N,\mathbb{R})}{\text{minimize}}\quad
& F(\bm{D})=\mathbb{E}_{\bm{y}\sim P}[-\|\bm{D}^{\text T}\bm{y}\|^3_3].
\\
\end{aligned}
\end{equation}
 The choice of  $-\Vert\cdot\Vert^3_3$ is inspired by the recent result  that minimizing the negative $p$-th power of the $\ell_p$-norm ($p\in\mathbb{N},\quad p>2$) with the unit $\ell_2$-norm constraint leads to sparse (or spiky) solutions \cite{zhai2020complete,shen2020complete,9470930}. An illustration is given in Fig. \ref{fig:lpball}. 
  \begin{figure}
 	\centering
 	\includegraphics[width=0.65\linewidth]{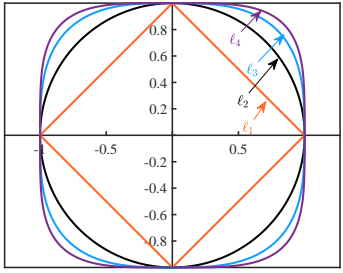}
 	\caption{Unit spheres of the of $\ell_p$ in $\mathbb{R}^2$, where $p = 1, 2, 3, 4$. The sparsest points on the unit $\ell_2$-norm-sphere, e.g., points $(0, 1),(0,-1),(1, 0)$ and $(-1,0)$ in $\mathbb{R}^2$, have the largest $\ell_p$-norm ($p\in\mathbb{N},\quad p>2$).}
 	\label{fig:lpball}
 \end{figure}
 Compared to the widely used $\ell_1$-norm,  the negative $\ell_p$-norm formulation allows a convex relaxation of the orthogonal constraint. Hence it provides flexibility for the algorithm design, as we will illustrate in Section \ref{sec:relaxprob}. Also, the differentiability of this formulation enables a faster convergence of algorithms. In this paper,  we choose $p=3$ since the sample complexity and the total computation complexity achieve the minimum  when $p=3$  among all the choices of $p \quad(p\in\mathbb{N}, p>2)$ for maximizing  the $\ell_p$-norm over the orthogonal group \cite{shen2020complete}.

 \subsection{Orthogonal Constraint Relaxation}\label{sec:relaxprob}
After determining the sparsity-promoting function, to facilitate an efficient online solver with a convergence guarantee,   we  propose a convex  relaxation of the orthogonal constraint in Problem (\ref{eq:DLonlinel31}) based on the following Lemma \ref{lem:conv}.
\begin{lemma}[\protect{\cite[3.4]{journee2010generalized}} Convex Hull of the Orthogonal Group]\label{lem:conv}
	 The convex hull of the orthogonal group is the (closed) unit spectral ball:
	 $$ conv(\mathbb{O}(N,\mathbb{R})) = \mathbb{B}_{sp}(N,\mathbb{R}),$$
	 where $\mathbb{B}_{sp}(N,\mathbb{R}):=\{\bm{X} \in\mathbb{R}^{N\times N}: \|\bm{X}\|\leq1\}$ is the unit spectral ball.
  \end{lemma}
From Lemma \ref{lem:conv}, we know that the unit spectral ball is the  minimal convex set containing the orthogonal group. Then we propose the following relaxed problem for the online ODL:

 \begin{equation}
\label{eq:DLonlinel32}  
\boxed{\mathscr{P}:\underset{\bm{D}\in \mathbb{B}_{sp}(N,\mathbb{R})}{\text{minimize}}\quad
F(\bm{D})=\mathbb{E}_{\bm{y}\sim P}[-\|\bm{D}^{\text T}\bm{y}\|^3_3].}
\end{equation}

Problem $\mathscr{P}$ is a proper relaxation of  Problem (\ref{eq:DLonlinel31}) since all the optimal solutions of Problem (\ref{eq:DLonlinel31}) belong to the set of the optimal solutions of Problem $\mathscr{P}$ under some general  statistical model for  $\bm{y}$. We will show this relationship formally in the following.

  We first  introduce  the following definition of the {\em sign-permutation matrix} for characterizing the optimal solutions.
\begin{defn}
	\label{def:(Phase-permutation-Matrix)-The}(Sign-permutation Matrix)
	The $N$-dimensional sign-permutation matrix $\bm{\Xi}\in\mathbb{R}^{N\times N}$
	is defined as
	\begin{equation}
	\boldsymbol{\Xi}=\boldsymbol{\Sigma}\boldsymbol{\Pi},
	\end{equation}
	where $\boldsymbol{\Sigma}=diag(\pm \bm{1}_N)$
	with $\bm{1}_N$ the $N$-dimensional all-one vector, and $\boldsymbol{\Pi}=[\bm{e}_{\pi(1)},\bm{e}_{\pi(2)},\ldots,\bm{e}_{\pi(N)}]$,
	with $\bm{e}_{n}$ being a standard basis vector and $[\pi(1),\pi(2),\ldots,\pi(N)]$
	being any permutations of the $N$ elements.
\end{defn}
Next, the relationship between the  optimal solutions of  Problem (\ref{eq:DLonlinel31})  and  Problem $\mathscr{P}$  is formally presented in the following Theorem \ref{thm:relax}.
\begin{theorem}[Consistency of the Relaxation] \label{thm:relax}
	If $\bm{y}$ follows the distribution $P$ such that $\bm{y}=\bm{D}^{true}\bm{x}$ with $\bm{D}^{true}\in\mathbb{O}(N,\mathbb{R})$ and the entries of $\bm{x}$ being i.i.d  Bernoulli Gaussian,\footnote{The Bernoulli Gaussian is a typical statistical model for the sparse coefficient, and it is widely used for the analysis of dictionary learning \cite{sun2015complete1,bai2018subgradient,zhai2020complete}.} $\bm{x} \overset{i.i.d}{\sim}\mathcal{BG}(\theta)$, then the optimal solution of Problem (\ref{eq:DLonlinel31}) is 
	\begin{equation}\label{eq:opt}
	\bm{D}^{opt}=\bm{D}^{true}\boldsymbol{\Xi}^{\text T},
	\end{equation}
	which  is also an optimal solution of Problem $\mathscr{P}$.  In (\ref{eq:opt}), $\boldsymbol{\Xi}$ is a sign-permutation matrix.
\end{theorem}
\begin{proof}
	See Appendix \ref{proof:relax}.
\end{proof}
\begin{remark}
This consistency of the relaxation no longer holds when $-\Vert\cdot\Vert^3_3$ in  Problem $\mathscr{P}$ is replaced by the widely used sparsity-promoting function $\Vert\cdot\Vert_1$. When $\Vert\cdot\Vert_1$ is used, the optimal solution over the orthogonal group is still  $\bm{D}^{true}\boldsymbol{\Xi}^{\text T}$, but the optimal solution over the unit spectral ball is $\bm{0}$. {  Hence, the relaxation no longer maintains the optimality of the minimizers of the original problem if $\Vert\cdot\Vert_1$ is used. The minimizers of the relaxed Problem $\mathscr{P}$ form a larger set compared to the minimizers of Problem (\ref{eq:DLonlinel31}). Hence, Problem $\mathscr{P}$  maintains the optimality of the minimizers of Problem (\ref{eq:DLonlinel31}), but the minimizers of Problem $\mathscr{P}$  are  not necessarily  the minimizers of Problem (\ref{eq:DLonlinel31}). However,  we can show that if the dictionary in  Problem $\mathscr{P}$ is further restricted to be fullrank, then the minimizers of  Problem $\mathscr{P}$ are also the minimizers of Problem (\ref{eq:DLonlinel31}) (see Appendix \ref{proof:relax}). Fortunately, the full-rank condition    holds,  as shown from the extensive experiments  in Section \ref{sec:exp}. We believe the probability of encountering a rank-deficient instance is very low when adopting  algorithms  with full-rank dictionary initialization and  randomness in updating  the dictionary, e.g., the proposed Algorithm \ref{alg:genFWODL}.}
\end{remark}
After the convex relaxation, we then focus on solving  Problem $\mathscr{P}$, which is a nonconvex optimization problem over a convex set. The convex set has a key property: it contains the convex combination of  any two points in the set. That is to say, if we have $\bm{A}, \bm{B} \in \mathbb{B}_{sp}(N,\mathbb{R})$, then 
\begin{equation}\label{eq:cvxcomb}
\eta\bm{A}+(1-\eta)\bm{B}\in\mathbb{B}_{sp}(N,\mathbb{R}), \eta\in(0,1).
\end{equation}
This property enables an efficient online algorithm with a convergence guarantee, as we will illustrate in the next section. 

\section{Online Nonconvex Frank-Wolfe-Based Algorithm}\label{sec:alg}
In this section, we first outline the  proposed Frank-Wolfe-based algorithm, NoncvxSFW, for  general online convex-constraint nonconvex problems, and then specialize it to solve Problem $\mathscr{P}$.

\subsection{NoncvxSFW for General Online Non-Convex Optimization}\label{sec:generalSFW}
To solve a general online convex-constraint nonconvex problem,
\begin{equation}\label{eq:genNoncvx}
\underset{\bm{X}\in  \underbrace{\mathcal{C}}_{\text{convex}}}{\text{min} } \quad
F_{gen}(\bm{X})=\mathbb{E}_{\bm{y}\sim P}[\underbrace{f(\bm{X},\bm{y})}_{\text{nonconvex in $\bm{X}$}}],
\end{equation}
we require the algorithm to have the following properties:
\begin{itemize}
	\item {\em Computational Efficiency}: Efficient per-iteration computation.
	\item {\em Theoretical Effectiveness}: Theoretical guarantee of convergence to a stationary point.
\end{itemize}
To fulfill the above properties, we propose  NoncvxSFW, as shown in  Algorithm \ref{alg:genFW}, which is a variant of the Stochastic Frank-Wolfe method (SFW) in \cite{mokhtari2020stochastic}. The SFW method and the
corresponding analysis  can only be applied to solve
convex problems. However,  NoncvxSFW and the analysis we
propose in this paper are also applicable to nonconvex problems. In the following, we will elaborate on how the proposed NoncvxSFW algorithm satisfies the required properties.
\begin{algorithm}\label{alg:genFW}
	\KwData{$\{\bm{Y}_t\}_{t=1}^\infty$ with $\bm{Y}_t=[\bm{y}^1_t,\ldots,\bm{y}_t^{M_t}]$}
	\KwResult{$\{\bm{X}_{t}\}_{t=1}^\infty$ }
	Initialization: { $\bm{G}_0=\bm{0}$ and random $\bm{X}_0 \in \mathcal{C}$ }\\
	\For{$ t = 1,2,\ldots$}{
		$\rho_t=4(t+1)^{-1/2}$, $\gamma_t=2(t+2)^{-3/4}$\\
		1. Gradient Approximation: $\bm{G}_t=(1-\rho_t)\bm{G}_{t-1}+\frac{\rho_t}{M_t}\sum_{j\in[M_t]}\nabla f(\bm{X}_{t-1},\bm{y}^j_t)$\\
		2. LMO:
		$\bm{S}_t = \arg\min_{\bm{S}\in \mathcal{C}}\langle \bm{G}_t, \bm{S} \rangle$\\
		3. Variable Update:
		$\bm{X}_{t} = \mathcal{P}[(1-\gamma_{t})\bm{X}_{t-1}+\gamma_{t}\bm{S}_t]$.
	}
	\caption{NoncvxSFW for General Nonconvex Problem}
\end{algorithm}

\subsubsection{Computational Efficiency of General Non-Convex Optimization}
Algorithm \ref{alg:genFW} comprises  three main steps. 
\begin{itemize}
	\item Step 1 ({\em Gradient Approximation}) approximates the true gradient $\nabla F_{gen}(\bm{X})$ with  $\bm{G}_t$ in a recursive way. In the calculation, $M_t$ can be fixed along all $t$. Hence, compared to the  methods in \cite{hazan2016variance} and \cite{reddi2016stochastic} that require  an increasing number of stochastic gradient evaluations as the number of iterations $t$ grows,  NoncvxSFW is more computationally efficient. 
	\item Step 2   ({\em LMO}) is a procedure to handle the constraint. It can be regarded as solving a linear approximation of the objective function over the constraint set $\mathcal{C}$ using the approximated gradient produced by Step 1. Compared to  the Quadratic Minimization Oracle (QMO) in the proximal-based methods, e.g., ProxSGD \cite{ghadimi2016mini}, the LMO can be more computationally efficient for many constraint sets, such as the trace norm and the $\ell_p$ balls \cite{jaggi2013revisiting}. 
	\item Step 3 ({\em Variable Update}) updates the variable $\bm{X}_{t}$ by  a simple convex combination of $\bm{S}_t\in\mathcal{C}$ and $\bm{X}_{t-1}\in\mathcal{C}$, $\mathcal{P}[\bm{X}_{t} = (1-\gamma_{t})\bm{X}_{t-1}+\gamma_{t}\bm{S}_t]$, where $\mathcal{P}[\bm{X}]$ is any operation that satisfies $\mathcal{P}[\bm{X}] \in \mathcal{C}$ and $F_{gen}(\mathcal{P}[\bm{X}])\leq F_{gen}(\bm{X})$. According to (\ref{eq:cvxcomb}), we have $\bm{X}_{t} = (1-\gamma_{t})\bm{X}_{t-1}+\gamma_{t}\bm{S}_t\in\mathcal{C}$. This step only requires  the output from Step 2 and the variable of  the last iteration and  automatically ensures the feasibility of the output $\bm{X}_t$.
\end{itemize}
	\subsubsection{Theoretical Effectiveness of General Non-Convex Optimization}

To show the convergence of the NoncvxSFW, we first introduce the following Frank-Wolfe gap as the measure for the first-order stationarity.
\begin{defn}[Frank-Wolfe Gap]
	The Frank-Wolfe gap at the $t$-th iteration, $g^{gen}_t$, is defined as
	\begin{equation}
	g^{gen}_{t}:=\underset{\bm{S}\in\mathcal{C}}{\max}\langle-\nabla F_{gen}(\bm{X}_{t-1}), \bm{S}-\bm{X}_{t-1}\rangle.
	\end{equation}
\end{defn}
The Frank-Wolfe gap is a valid first-order stationary measure because of the following Lemma \ref{lem:FWgap}.
\begin{lemma}[Frank-Wolfe Gap is a Measure for Stationarity]\label{lem:FWgap}
	A point $\bm{X}_{t-1}$	is a stationary point for the optimization problem (\ref{eq:genNoncvx}) if and only if $g^{gen}_t = 0$.
\end{lemma}
\begin{proof}
	See Appendix \ref{proof:gap}.
\end{proof}
Then we assume the following conditions hold for the general problem (\ref{eq:genNoncvx}).
\begin{assum}\label{assum1}
	$\quad$
	\begin{enumerate}
		\item (Bounded Constraint Set)	The constraint set $\mathcal{C}$ is   bounded with diameter $diam(\mathcal{C})$ in terms of the Frobenius norm for matrices, i.e.,
		\begin{equation}
		\|\bm{X}-\bm{Y}\|_F \leq diam(\mathcal{C}),\quad \forall  \bm{X},\bm{Y}\in\mathcal{C}.
		\end{equation}
		\item 	(Lipschitz Smoothness) $ F_{gen}(\bm{X})$ is $L$-smooth over the  set $\mathcal{C}$, i.e.,
		\begin{equation}
		\begin{aligned}
			\|\nabla  F_{gen}(\bm{X})-\nabla  F_{gen}(\bm{Y})\|_F\leq L \|\bm{X}-\bm{Y}\|_F,\\ \forall  \bm{X},\bm{Y}\in\mathcal{C}.
		\end{aligned}
		\end{equation}
		\item  (Unbiased Mini-batch Gradient) The mini-batch gradient $\frac{1}{M_t}\sum_{j\in[M_t]}\nabla f(\bm{X}_{t-1},\bm{y}^j_t)$ is an unbiased estimation of the true gradient $\nabla  F_{gen}(\bm{X})$, i.e.,
		\begin{equation}
		\mathbb{E}\Big[\frac{1}{M_t}\sum_{j\in[M_t]}\nabla f(\bm{X}_{t-1},\bm{y}^j_t)\Big]=\nabla  F_{gen}(\bm{X}),\quad \forall t.
		\end{equation}
		\item 	(Bounded Variance of the Mini-batch Gradient) The variance of the mini-batch gradient $\frac{1}{M_t}\sum_{j\in[M_t]}\nabla f(\bm{X}_{t-1},\bm{y}^j_t)$ is bounded; i.e., for all $t$,  we have
		\begin{equation}
		\mathbb{E}\Big[\|\frac{1}{M_t}\sum_{j\in[M_t]}\nabla f(\bm{X}_{t-1},\bm{y}^j_t)-\nabla  F_{gen}(\bm{X})\|_F^2\Big]
		\leq\frac{V}{M_t}.	
		\end{equation}
	\end{enumerate}
\end{assum}
The above assumptions ensure  the convergence of  NoncvxSFW, which is formally stated in the following Theorem \ref{thm:convgen}.
\begin{theorem}[Convergence of  NoncvxSFW for the General Nonconvex Problem]\label{thm:convgen}
If the conditions in Assumption \ref{assum1} hold, using  Algorithm \ref{alg:genFW} to solve the general problem (\ref{eq:genNoncvx}), we have that the expected Frank-Wolfe gap converges to zero, in the sense that 
	\begin{equation}
	\begin{aligned}
    &\underset{1< s\leq t} {\inf} \mathbb{E}\Big[g^{gen}_s\Big]\\
	\leq & \frac{c_1 (\sqrt{\max\{C_0,C_1\}}diam(\mathcal{C})+L diam(\mathcal{C})^2)\ln(t+2)}{ (t+3)^{\frac{1}{4}}},
	\end{aligned}
	\end{equation}
	where $C_0 = \|\nabla  F_{gen}(\bm{X}_0)\|^2_F$, $C_1 = \frac{4V}{\underset{t}{\min}\{M_t\}}+2L^2  diam(\mathcal{C})^2$ and $c_1$ is some positive constant. 
	In other words, Algorithm \ref{alg:genFW} is guaranteed to converge to a stationary point of the general problem (\ref{eq:genNoncvx}) at a rate of $\mathcal{O}(\ln(t)/t^{1/4})$ in expectation. 
\end{theorem}
\begin{proof}
	See Appendix \ref{proof:conv1}.
\end{proof}

\subsection{NoncvxSFW for the Proposed Online ODL Problem}
In this section, we will apply the proposed NoncvxSFW to solve the proposed online ODL Problem $\mathscr{P}$. The algorithm is summarized in Algorithm \ref{alg:genFWODL}. 
\begin{algorithm}\label{alg:genFWODL}
	\KwData{$\{\bm{Y}_t\}_{t=1}^\infty$ with $\bm{Y}_t=[\bm{y}^1_t,\ldots,\bm{y}_t^{M_t}]$}
	\KwResult{$\{\bm{D}_{t}\}_{t=1}^\infty$ }
	Initialization: { $\bm{G}_0=\bm{0}$ and {random $\bm{D}_0 \in \mathbb{O}(N,\mathbb{R})\subset \mathbb{B}_{sp}(N,\mathbb{R})$}}\\
	\For{$ t = 1,2,\ldots$}{
		$\rho_t=4(t+1)^{-1/2}$, $\gamma_t=2(t+2)^{-3/4}$\\
		1. Gradient Approximation: $\bm{G}_t=(1-\rho_t)\bm{G}_{t-1}+\frac{\rho_t}{M_t}\sum_{j\in[M_t]}-\nabla \|\bm{D}_{t-1}^{\text T}\bm{y}^j_t\|^3_3$\\
		2. LMO:
		$\bm{U},\bm{\Sigma},\bm{V}^{\text T}=\text{SVD}(-\bm{G}_t)$\\
		$\quad\quad\quad\quad \bm{S}_t = \bm{U}\bm{V}^{\text T}$\\
		3. Variable Update:
		$\bm{D}_{t} = Polar((1-\gamma_{t})\bm{D}_{t-1}+\gamma_{t}\bm{S}_t)$.
	}
	\caption{NoncvxSFW for the proposed Online ODL problem}
\end{algorithm}
 Similar to Section \ref{sec:generalSFW}, we will illustrate the NoncvxSFW for the proposed Online ODL problem  from the computational and theoretical aspects.
\subsubsection{Computational Efficiency of the Proposed Online ODL Problem}
{  When adopting  NoncvxSFW for Problem $\mathscr{P}$, we can obtain a computationally efficient ODL algorithm whose complexity will not increase as $t$ grows. 
	
	 Step 1 ({\em Gradient Approximation})  remains unchanged in Algorithm \ref{alg:genFW}. The sampled gradient  can be expressed as $-\nabla \|\bm{D}^{ T}\bm{y}^j_t\|^3_3= -\bm{y}_t^j(|(\bm {D}^{(t-1)})^{\text T}\bm{y}_t^j|\odot (\bm {D}^{(t-1)})^{\text T}\bm{y}_t^j)^{\text T}$. Hence, at each iteration,  the time complexity and  memory complexity of Step 1  are  $\mathcal{O}(N^2M_t)$ ( $M_t$ can be fixed along all $t$) and $\mathcal{O}(N^2)$, respectively.

 Step 2 ({\em LMO})  is calculated based on the following Lemma \ref{lem:LMO}.
\begin{lemma}[LMO for the Unit Spectral Ball]\label{lem:LMO}
The minimum value  of $\langle \bm{G}, \bm{S} \rangle,  \forall \bm{S}\in \mathbb{B}_{sp}(N,\mathbb{R})$ is the nuclear norm of $-\bm{G}$, i.e.,
\begin{equation}
\min_{\bm{S}\in \mathbb{B}_{sp}(N,\mathbb{R})}\langle \bm{G}, \bm{S} \rangle = \|-\bm{G}\|_*.
\end{equation}
The minimum is achieved when $\bm{S}$ belongs to the subdifferential of $\|-\bm{G}\|_*$, i.e.,
\begin{equation}\label{LMOS}
\begin{aligned}
\bm{S}^* =  &\arg\min_{\bm{S}\in \mathbb{B}_{sp}(N,\mathbb{R})}\langle \bm{G}, \bm{S} \rangle = \bm{U}\bm{V}^{\text T} \in \partial\|-\bm{G}\|_*,
\end{aligned}
\end{equation}
where 	$\bm{U},\bm{\Sigma},\bm{V}^{\text T}=\text{SVD}(-\bm{G})$.
\end{lemma}
\begin{proof}
We can prove Lemma 3 by simply using the definition of the dual norm and the subdifferential of the norm.\footnote{ The detailed deduction can be found in \url{https://stephentu.github.io/blog/convex-analysis/2014/10/01/subdifferential-of-a-norm.html}.} 
\end{proof} Using the LMO to deal with the constraint is much more computationally friendly than  using the QMO  in the proximal-based method. In the QMO, the proximal operator for the matrix spectral norm requires a proximal operator for the $\ell_{\infty}$-norm of the singular vector, which has no closed-form  solution \cite{duchi2008efficient}. The calculation of  the LMO can be simplified via the fact that $-\bm{G}=\bm{U}\bm{\Sigma}\bm{V}^{\text T}=\bm{U}\bm{V}^{\text T}\bm{V}\bm{\Sigma}\bm{V}^{\text T} =\bm{S}^*\bm{V}\bm{\Sigma}\bm{V}^{\text T}$. This indicates that $\bm{S}^*$ can be calculated directly from the  polar decomposition of $-\bm{G}$, which has many efficient calculations \cite{higham1994parallel}.  Using  Coppersmith-Winograd matrix multiplication \cite{spatial} in the calculation of polar decomposition,  the time complexity and  memory complexity of Step 2  are $\mathcal{O}(N^{2.38})$  and $\mathcal{O}(N^2)$, respectively.

In Step 3 ({\em Variable Update}), we further adopt  polar decomposition,  inspired by its projection property \cite[Proposition 3.4]{absil2012projection}, which has  $\mathcal{O}(N^{2.38})$  time complexity and  $\mathcal{O}(N^2)$ memory complexity.} The following lemma shows that this is a valid specification of  $\mathcal{P}$ in Algorithm \ref{alg:genFW}.
\begin{lemma}(Validation of Variable Update Step)\label{lem:polar}
Let $\bm{X}\in \mathbb{B}_{sp}(N,\mathbb{R})$. Then we have
$Polar(\bm{X})\in \mathbb{B}_{sp}(N,\mathbb{R})$ and $F(Polar(\bm{X}))\leq F(\bm{X})$.
\end{lemma}
\begin{proof}
	See Appendix \ref{proof:polar}.
\end{proof}
\subsubsection{Theoretical Effectiveness of the Proposed Online ODL Problem}
In this subsection, we will adapt the convergence theory in Theorem \ref{thm:convgen} to the proposed online ODL problem. We first show in the following Lemma \ref{lem:check} that the conditions in  Assumption \ref{assum1} are satisfied by Problem $\mathscr{P}$. 

\begin{lemma}[The ODL Problem Satisfies the Convergence Condition]\label{lem:check}
	If $\bm{y}$ follows the distribution $P$ such that,  for all $t$ and all $j$, $\bm{y}^j_t=\bm{D}^{true}\bm{x}^j_t$, with $\bm{D}^{true}\in\mathbb{O}(N,\mathbb{R})$ and the entries of $\bm{x}^j_t$ being i.i.d  Bernoulli Gaussian,  $\bm{x}^j_t \overset{i.i.d}{\sim}\mathcal{BG}(\theta)$, then Problem $\mathscr{P}$ satisfies the conditions in Assumption \ref{assum1}. Specifically, we have:
	\begin{enumerate}
		\item (Bounded Constraint Set)	
		\begin{equation}
		\begin{aligned}
			&\|\bm{D}_1-\bm{D}_2\|_F \leq \sqrt{2N}, \forall  \bm{D}_1,\bm{D}_2\in\mathbb{B}_{sp}(N,\mathbb{R}).
		\end{aligned}
		\end{equation}
		\item 	(Lipschitz Smoothness) $F(\bm{D})$ is $L$-smooth over the  set $\mathbb{B}_{sp}(N,\mathbb{R})$, i.e.,
	\begin{equation}
	\begin{aligned}
	&\|\nabla F(\bm{D}_1)-\nabla F(\bm{D}_2)\|_F\\
	\leq& \sqrt{\frac{2}{\pi}}N^{3/2}(N+1)\theta \|\bm{D}_1-\bm{D}_2\|_F,\\
	&\forall  \bm{D}_1,\bm{D}_2\in\mathbb{B}_{sp}(N,\mathbb{R}).
	\end{aligned}
	\end{equation}
	\item  (Unbiased Mini-batch Gradient) The mini-batch gradient $\frac{1}{M_t}\sum_{j\in[M_t]}-\nabla \|\bm{D}^{\text T}\bm{y}^j_t\|^3_3$ is an unbiased estimation of the true gradient $\nabla F(\bm{D})$, i.e.,
	\begin{equation}
	\mathbb{E}\Big[\frac{1}{M_t}\sum_{j\in[M_t]}-\nabla \|\bm{D}^{\text T}\bm{y}^j_t\|^3_3\Big]=\nabla F(\bm{D}),\quad \forall t.
	\end{equation}
	\item 	(Bounded Variance of the Mini-batch Gradient) The variance of the mini-batch gradient $\frac{1}{M_t}\sum_{j\in[M_t]}-\nabla \|\bm{D}^{\text T}\bm{y}^j_t\|^3_3$ is bounded; i.e., for all $t$, we have
	\begin{equation}
	\mathbb{E}\Big[\|\frac{1}{M_t}\sum_{j\in[M_t]}-\nabla \|\bm{D}^{\text T}\bm{y}^j_t\|^3_3-\nabla F(\bm{D})\|_F^2\Big]\leq\frac{3\theta N^2}{M_t}.
	\end{equation}
	\end{enumerate}
\end{lemma}
\begin{proof}
	See Appendix \ref{proof:ODLcond}.
\end{proof}
Based on Lemma \ref{lem:check} and Theorem \ref{thm:convgen}, we have the convergence result for Algorithm \ref{alg:genFWODL}, as shown in the following Theorem \ref{thm:convODL}.

\begin{theorem}[Convergence of  NoncvxSFW for the Proposed Online ODL Problem]\label{thm:convODL}
 Using  Algorithm \ref{alg:genFWODL} to solve Problem $\mathscr{P}$, we have that the expected Frank-Wolfe gap
 	$$\mathbb{E}\Big[g_{t}\Big]:=\mathbb{E}\Big[\underset{\bm{S}\in\mathbb{B}_{sp}(N,\mathbb{R})}{\max}\langle-\nabla F(\bm{D}_{t-1}), \bm{S}-\bm{D}_{t-1}\rangle\Big]$$   converges to zero, in the sense that 
	\begin{equation}
	\footnotesize
	\begin{aligned}
	&\underset{1< s\leq t} {\inf} \mathbb{E}\Big[g_s\Big]\\
	\leq &  \frac{c_2\Big(\sqrt{C_0N+\frac{\theta N^3}{\underset{t}{\min}\{M_t\}}+\theta^2N^5(N+1)^2 }+\theta N^{\frac{5}{2}}(N+1)\Big)\ln(t+2)}{ (t+3)^{\frac{1}{4}}},
	\end{aligned}
	\end{equation}
	where $C_0 = \|\nabla  F(\bm{D}_0)\|^2_F$ and  $c_2$ is a positive constant. 
\end{theorem}
\begin{proof}
	The Proof can be made by substituting the results in Lemma \ref{lem:check} into Theorem \ref{thm:convgen}.
\end{proof}
\begin{remark}[Impact of the Key System Parameters]
	Theorem \ref{thm:convODL} suggests that a larger value of the smallest mini-batch size $\underset{t}{\min}\{M_t\}$, a smaller value of dictionary size $N$, and a smaller sparsity level $\theta$ (data becomes more sparse with a smaller $\theta$) will lead to a faster convergence speed. These conclusions are consistent with the simulation results in  Section \ref{sec:syndata}. Though the above theorem only proves the convergence to stationary points,  we have observed in experiments that the algorithm actually converges to the global optimum under very broad conditions, as shown in Section \ref{sec:syndata}. Similar phenomena have also appeared in many other works on offline ODL \cite{zhai2020complete,shen2020complete,9470930,bai2018subgradient,sun2015complete1}.
\end{remark}

\section{ Application Examples}\label{sec:app}
In this subsection, we give two important application examples
where  the online ODL method should be adopted.
\subsubsection{Example 1 (Online Data Compression on Edge Devices in the IoT Network \cite{lu2020adaptively})}
Consider an IoT network architecture shown in Fig. \ref{fig:edge},
where the data are collected from smart objects, such as wearables and industrial sensor devices, and are sent periodically to an edge device using short-range communication protocols (e.g., WiFi and Bluetooth). The edge device is responsible for  low-level processing, filtering, and sending the data to the cloud. 
\begin{figure}[htbp]
	\centering
	\includegraphics[width=1.0\linewidth]{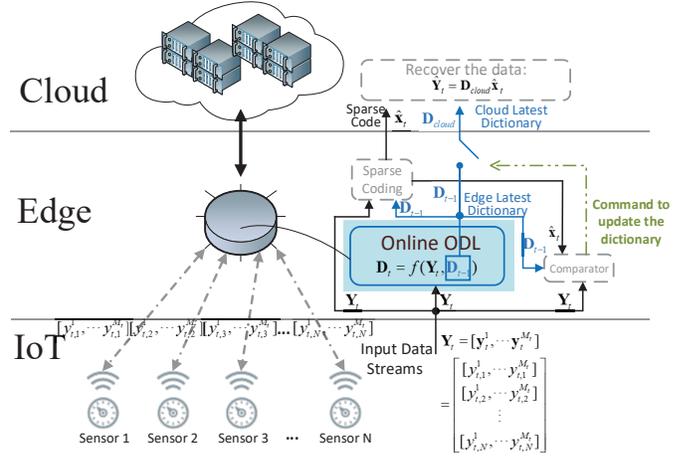}
	\caption{ Illustration of edge data compression in an IoT network using the proposed online ODL scheme.}
	\label{fig:edge}
\end{figure}
We assume that an edge device is connected to a total number of $N$ geographically distributed  IoT sensors. When  sensor measurements (temperature, humidity,
concentration, etc.) are required by the cloud from the edge for the  data analytics, the edge device transmits a compressed version of the data to save  communication resources. At the $t$-th time slot,  the ${M_t}$ samples of the sensor  measurements from all $N$ sensors, $\bm{Y}_t=[\bm{y}^1_t,\ldots,\bm{y}_t^{M_t}]\in\mathbb{R}^{N\times {M_t}}$, are transmitted to an edge  device.	 When the $j$-th sample from all $N$ sensors, i.e.,  $\bm{y}^j_t,j\in[M_t]$, is required by the cloud, the edge data compression is executed using the following steps:
\begin{enumerate}
	\item Preprocessing ({\em Sparse Coding on the Edge}): The edge device calculates the sparse code $\hat{\bm{x}}^j_t$ based on the  {\em latest edge  dictionary} $\bm{D}_{t-1}$ and the input  $\bm{y}^j_t$.
	\item Preprocessing ({\em Transmission Content Decision on the Edge}):
	Upon obtaining  the sparse code $\hat{\bm{x}}^j_t$, the edge device calculates the error between $\bm{y}^j_t$ and $\bm{D}_{cloud}\hat{\bm{x}}^j_t$ using a certain error metric $l(\bm{y}^j_t,\bm{D}_{cloud}\hat{\bm{x}}^j_t)$, where $\bm{D}_{cloud}$ is a local copy of  the {\em latest cloud  dictionary}  in the cloud. Then, the edge decides the content to transmit:
	\begin{itemize}
		\item If the error metric $l(\bm{y}^j_t,\bm{D}_{cloud}\hat{\bm{x}}^j_t)$ is larger than a predetermined threshold,   the edge device  updates its local cloud dictionary copy as $\bm{D}_{cloud}=\bm{D}_{t-1}$, and transmits the updated $\bm{D}_{cloud}$ to the cloud in a compressed format. It also  transmits the sparse code  $\hat{\bm{x}}^j_t$ to the cloud in a compressed format;
		\item Otherwise, the edge  transmits the sparse code  $\hat{\bm{x}}^j_t$ to the cloud in a compressed format.
	\end{itemize}
	\item Core Procedure ({\em Online ODL on the Edge}): The edge device runs the online ODL method to produce $\bm{D}_{t}$ using $\bm{D}_{t-1}$ and the input $\bm{Y}_t$. 
	\item Postprocessing ({\em Sensor Data Recovery on the Cloud}): The cloud recovers the required data $\bm{y}^j_t$ by $\hat{\bm{y}}_t^j=\bm{D}_{cloud}\hat{\bm{x}}^j_t$, where $\hat{\bm{x}}^j_t$ is the sparse code received from the edge and $\bm{D}_{cloud}$ is the latest dictionary in the cloud.
\end{enumerate}

The proposed ODL module produces the $\bm{D}_t$, $\bm{D}_{t-1}$ and $\bm{D}_{cloud}$, which play critical roles in the above example of data compression on  IoT edge devices.

\subsubsection{Example 2 (Real-time Novel Document Detection \cite{kasiviswanathan2012online})} Novel document detection can be used to find  breaking news or emerging topics on social media. In this application, $\bm{Y}_t=[\bm{y}^1_t,\ldots,\bm{y}_t^{M_t}]\in\mathbb{R}^{N\times {M_t}}$  denotes the mini-batch of documents arriving at time $t$, where each column of $\bm{Y}_t$ represents a document at that time, as shown in Fig. \ref{fig:detector}.
\begin{figure}
	\centering
	\includegraphics[width=0.7\linewidth]{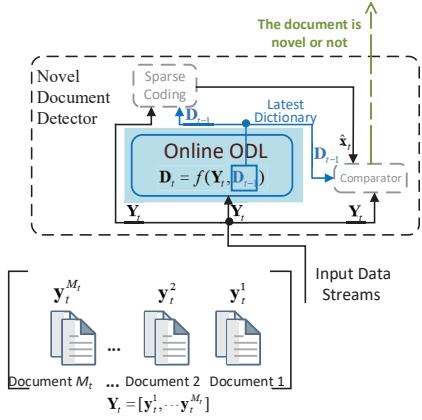}
	\caption{ Illustration of  real-time novel document detection using the proposed online ODL scheme.}
	\label{fig:detector}
\end{figure}
Each document is represented by a conventional vector space model such as TF-IDF \cite{cambridge2009online}. For the  mini-batch of documents $\bm{Y}_t$ arriving at time $t$, the novel document detector operates using the following steps:
\begin{enumerate}
	\item Preprocessing ({\em Sparse Coding}): For all $\bm{y}^j_t$ in $\bm{Y}_t$,  the detector calculates the sparse code $\hat{\bm{x}}^j_t$ based on the {\em latest  dictionary} $\bm{D}_{t-1}$ and the input $\bm{y}^j_t$.
	\item Preprocessing ({\em Novel Document Detection}): For all $\bm{y}^j_t$ in $\bm{Y}_t$, the detector calculates the error between $\bm{y}^j_t$ and $\bm{D}_{t-1}\hat{\bm{x}}^j_t$ with an error metric $l(\bm{y}^j_t,\bm{D}_{t-1}\hat{\bm{x}}^j_t)$. 
	\begin{itemize}
		\item If the error  $l(\bm{y}^j_t,\bm{D}_{t-1}\hat{\bm{x}}^j_t)$ is larger than some predefined threshold,   the detector marks the document $\bm{y}^j_t$ as {\em novel};
		\item Otherwise,  the detector marks the document $\bm{y}^j_t$ as {\em non-novel}.
	\end{itemize}
	\item Core Procedure ({\em Online ODL}): The detector runs the online ODL method to produce the new dictionary $\bm{D}_t$ using $\bm{D}_{t-1}$ and the input $\bm{Y}_t$. 
\end{enumerate}

\section{Experiments}\label{sec:exp}
This section provides experiments demonstrating the effectiveness and the
efficiency of our scheme compared to the state-of-the-art prior works. All the experiments 
are conducted in Python 3.7 with a 3.6 GHz Intel Core I7 processor.
\subsection{List of Baseline Methods}\label{subsec:list} The baseline methods are listed as follows: \footnote{{The proposed method, Baselines 1 and 2 have no hyperparameters. For Baselines $3\sim 5$,   grid search is adopted for tuning the hyperparameters.  The grids are drawn around the hyperparameter values given in the baseline papers \cite{mairal2009online,spatial,akhtar2017nonparametric}. In the experiments with real-world sensor data, we extend the grids for $\lambda$ to $[1,20]$ for Baselines 3 and 4 when $\eta_{0}=2,8,10$. We pick the hyperparameter with the best performance in  hindsight for the online method, Baseline 3, in both the synthetic data and the real-world data experiments. For the offline methods, Baselines 4 and 5, we adopt the walk-forward validation method \cite{zbikowski2015using} with  $4593$ testing data as the holdout data to pick the hyperparameters in the real-world data experiment.}}
\begin{itemize}
	\item \textbf{Baseline 1} (SFW) \cite{mokhtari2020stochastic}: { This baseline adopts the recently proposed SFW algorithm to solve the online ODL problem $\mathscr{P}$ in (\ref{eq:DLonlinel32}). We compare the proposed NoncvxSFW to this baseline in terms of the convergence property to show the effectiveness and efficiency of the proposed NoncvxSFW algorithm for solving the online ODL problem $\mathscr{P}$.}
	
	\item \textbf{Baseline 2} ($\ell_4$\_NoncvxSFW) \cite{zhai2020complete}: In this baseline, we replace the sparsity-promoting function $-\Vert\cdot\Vert^3_3$ in the ODL problem $\mathscr{P}$ with $-\Vert\cdot\Vert^4_4$. Then, we solve the problem by the NoncvxSFW algorithm. This baseline is adopted to demonstrate the advantage of the choice of the negative $\ell_3$-norm objective in the online ODL formulation.
	\item\textbf{Baseline 3} (Online AODL) \cite{mairal2009online}: This baseline alternately learns the sparse code and the dictionary by solving the following optimization problem in an online manner:
	 \begin{equation}\label{eq:AODLonline}  
	\begin{aligned}
	\underset{\bm{D}\in\mathbb{R}^{N\times N},\{\bm{x}_t\in\mathbb{R}^N\}_{t=1}^T}{\text{minimize}}\quad
	& \sum_{t=1}^T\|\bm{y}_t-\bm{D}\bm{x}_t\|_F^2+\lambda\|\bm{x}_t\|_1.
	\end{aligned}
	\end{equation}
	The online AODL is introduced to show the advantage of the online ODL scheme. The Python SPAMS toolbox is used to implement this baseline.\footnote{Python code and  documents are available at  \url{http://spams-devel.gforge.inria.fr/downloads.html}.}
	
	{ 	\item\textbf{Baseline 4} (Offline-DeepSTGDL) \cite{spatial}: This baseline  is a recently proposed offline dictionary learning method for behind-the-meter load and photovoltaic forecasting. In DeepSTGDL, a deep encoder first transforms the load measurements into a latent space which captures the spatiotemporal patterns of the data. Then, a spatiotemporal dictionary and a sparse code  are alternately learned to capture the significant spatiotemporal patterns for forecasting. 
		\item\textbf{Baseline 5} ({Offline-NCBDL}) \cite{akhtar2017nonparametric}: This baseline  is an offline nonparametric Bayesian approach in which the dictionary is inferred from a hierarchical Bayesian model based on the  Beta-Bernoulli process and Gibbs sampling. }   
\end{itemize}

\subsection{Experiments with Synthetic Data}\label{sec:syndata}
\subsubsection{Experiment Settings}
{ We evaluate the convergence property of 
	 different online dictionary
	learning methods with synthetic data. For all the experiments, we conduct  $100$ independent  Monte Carlo trials. At the $l$-th trial, we  generate the measurements $\bm{y}^j_t(l)={\bm{D}}^{true}(l)\bm{x}^j_t(l)\quad (j\in[M_t] ,1 \leq t \leq T)$, with the ground truth dictionary ${\bm{D}}^{true}(l)$ drawn uniformly randomly from the orthogonal group $\mathbb{O}(N,\mathbb{R})$, and with sparse signals $\bm{x}^j_t(l)\in\mathbb{R}^N$ drawn from an i.i.d. Bernoulli-Gaussian distribution, i.e.,  $\bm{x}^j_{t}(l) \overset{i.i.d}{\sim} \mathcal{BG}(\theta)$. For a fair comparison, all the methods share the same random initial point at each trial. Without loss of generality, we set $M_t=B,$ $\forall t$ and $T=3\times 10^3$ for data generation. 
	
	To evaluate the convergence property,  the  error metric at time index $t$  is calculated by
	\begin{equation}\label{eq:Merror}
	\text{Error}_t= \frac{1}{100}\sum_{l=1}^{100}|1-\frac{\|\bm{D}^{\text T}_t(l){\bm{D}}^{true}(l)\|_4^4}{N}|.
	\end{equation}
	This metric is an averaged  measure for the difference between the dictionary learning result and the ground truth dictionary, since  the true dictionary at the $l$-th trial will be perfectly recovered if  $\frac{\|\bm{D}^{\text T}_t(l){\bm{D}}^{true}(l)\|_4^4}{N}=1$ \cite{zhai2020complete,shen2020complete}. We compare the $\text{Error}_t$ versus the number of iterations of the proposed method and the online baseline methods as follows.} 
\subsubsection{Convergence Comparison with Different System Parameters}

In Fig. \ref{fig:convdiffb}, we show the convergence properties under different mini-batch sizes $B$ with a fixed dictionary size $N=10$ and sparsity level $\theta = 0.3$.  { We tune $\lambda$ for Baseline 3 at $\lambda=0.1$.} The results show that a larger mini-batch size $B$  can accelerate the convergence of all methods, and the proposed method with $B=10$ can converge to an error at $10^{-3}$ at around  the $1000$-th time index, which is faster than the baselines.\footnote{ If $t=\infty$, there should be no gap between the results for Baseline 2 and the proposed method under the same mini-batch size. However, in the simulation, $t=\infty$ is prohibitive. Therefore, when $t$ is finite, the difference of the curves comes from the fact that the $\ell_{3}$-norm-based formulation has a lower sample complexity than the $\ell_{4}$-norm-based formulation \cite{shen2020complete}. Hence, when we have a finite number of samples ($t$ is finite), the optimal point of the proposed formulation will be closer to the ground truth than the formulation in Baseline 2, and therefore shows a better performance.}
\begin{figure}[htbp]
	\centering
	\includegraphics[width=1.0\linewidth]{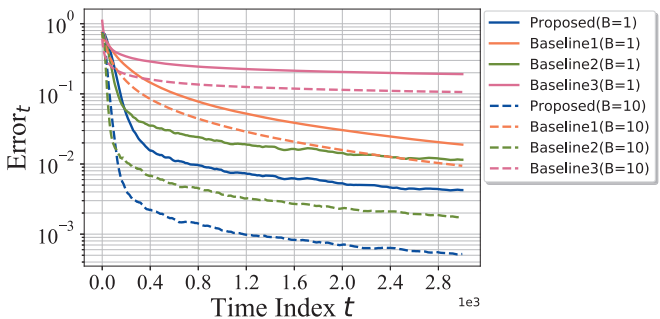}
	\caption{ Convergence comparison under different mini-batch sizes $B$ with  dictionary size $N = 10$ and sparsity level $\theta = 0.3$.}
	\label{fig:convdiffb}
\end{figure}

In Fig. \ref{fig:convdiffn}, the convergence curves  with different dictionary sizes $N$ are plotted. We fix the  mini-batch size at  $B=10$  and  the sparsity level at $\theta = 0.3$. { We tune $\lambda$ for Baseline 3 at $\lambda=0.1$.} The results show that  all the methods have a faster convergence rate under a  smaller dictionary  size $N$. In addition, the proposed  method   can achieve a smaller error with fewer number of iterations  than the baselines.
\begin{figure}[htbp]
	\centering
	\includegraphics[width=1.0\linewidth]{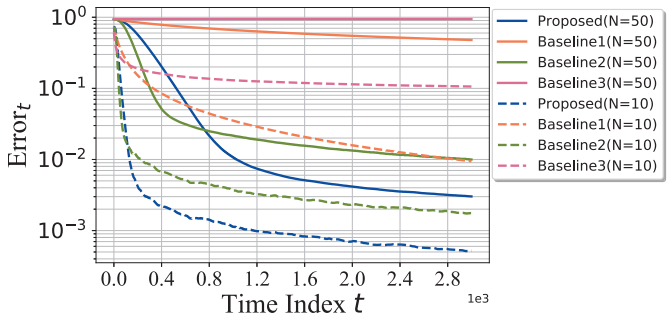}
	\caption{ Convergence comparison under different dictionary sizes $N$ with mini-batch size $B = 10$ and sparsity level $\theta = 0.3$.}
	\label{fig:convdiffn}
\end{figure}

In Fig. \ref{fig:convdifftheta}, we compare the convergence properties   with different sparsity levels $\theta$. The  mini-batch size  and  dictionary size are fixed at $B=10$ and $N=10$. { We tune $\lambda$ for Baseline 3. Specifically, we set $\lambda=0.1$ with $\theta =0.3$ and $\lambda=0.05$ with $\theta =0.5$.} The results show that  both the proposed method and the baselines  have a faster convergence rate with a  more sparse $\bm{x}$ (smaller $\theta$). Moreover, the proposed  method achieves $10^{-3}$ error at around the $2000$-th time index with $\theta = 0.5$, which is    faster than the baselines.

\begin{figure}[htbp]
	\centering
	\includegraphics[width=1.0\linewidth]{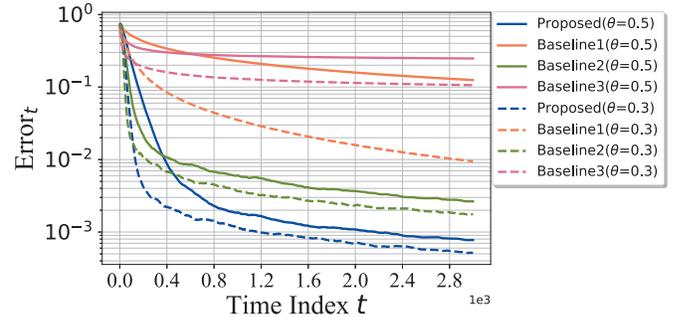}
	\caption{ Convergence comparison under different sparsity levels $\theta$ with dictionary size $N = 10$ and mini-batch size $B = 10$.}
	\label{fig:convdifftheta}
\end{figure}

\subsection{Experiments with a Real-World Sensor Data Set}

	 \subsubsection{Experiment Data Set} We evaluate the performance of the IoT sensor data compression task with different dictionary learning methods. The experiments are carried out on the {\em  Airly network air quality data set} \cite{poland}, which records  temperature, air pressure, humidity, and the concentrations of particulate matter from 00:00, Jan. 1, 2017 to 00:00, Dec. 25, 2017. The sensor readings are measured by a network of $N=56$ low-cost sensors located in Krakow, Poland, and each sensor has its own location.\footnote{Detailed location information with latitude and longitude can be found in the data set \cite{poland}.} There are $8593$ readings for each item from each sensor sampled per hour. In this work, we use the {\em temperature readings} from all the sensors as the input to the  dictionary learning schemes. Since there is a missing data issue in the raw data, we replace the missing data  with the mean readings over all the sensors at the times that data are missing. Fig. \ref{fig:sensorraw}  displays the  temperature readings that  the dictionary learning  schemes process.
	\begin{figure}
		\centering
		\includegraphics[width=0.85\linewidth]{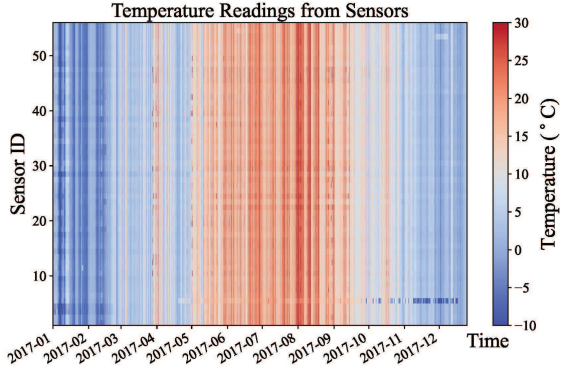}
			\caption{ $56\times 8593$ temperature readings from the {\em  Airly network air quality data set} \cite{poland}. The readings are from a  sensor network with $56$ sensors  deployed in different locations  in  Krakow, Poland.   The readings are collected every hour from Jan 1, 2017 to Dec. 25, 2017.   }
		\label{fig:sensorraw}
	\end{figure}

{ \subsubsection{Performance Metrics}
At each time $t$,  $M_t$ readings, $\bm{y}^j_t\in \mathbb{R}^{56}\quad (j\in[M_t] ,1 \leq t \leq T)$, are uploaded to  the dictionary learning processors, and will be approximated by $\tilde{\bm{y}}^j_t\quad (j\in[M_t] ,1 \leq t \leq T)$ through  calculations of the learned dictionary and the sparse code. To show  the compression performance of the dictionary learning methods, two performance metrics are calculated under different compression ratios.  
The compression ratio is defined as
\begin{equation}
\text{compression ratio}=\Big\lfloor\frac{56}{\eta_0}\Big\rfloor,
\end{equation}
where $\eta_0$ is the number of nonzero values in the sparse code.

The first performance metric is the RMSE, which is defined by 
\begin{equation}\label{eq:RSME}
\begin{aligned}
\text{RMSE}&=\sqrt{\frac{\sum^{T}_{t=0}\sum^{M_t}_{j=1}\|\tilde{\bm{y}}^j_t-\bm{y}^j_t\|_2^2}{\sum^{T}_{t=0}\sum^{M_t}_{j=1}\|\bm{y}^j_t\|_2^2}},\\
& \text{with}\quad \tilde{\bm{y}}^j_t = \phi (\bm{D}^{est}_t,\tilde{\bm{x}}^j_t), \|\tilde{\bm{x}}^j_t\|_0=\eta_0,
\end{aligned}
\end{equation}
where $\phi(\cdot)$ is  calculations of the dictionary and the sparse code  determined by each method, as we will specify in Section \ref{sec:proce}, $\bm{D}^{est}_t$ is the dictionary for compression at time $t$, and $\tilde{\bm{x}}^j_t$ is the sparse code  for the $j$-th reading at time $t$ with  $\eta_0$  nonzero values.

To provide more persuasive  results, we have also calculated the HLN-corrected  Diebold-Mariano (HLNDM) \cite{harvey1997testing} test results to quantitatively evaluate  the compression accuracy of different methods from a statistical point of view. Specifically, the HLNDM statistic is calculated by 
\begin{equation}
\begin{aligned}\label{eq:DM}
HLNDM=\sqrt{\frac{T-1-2h+h(h-1)}{T}}\frac{\bar{d}}{\sqrt{\frac{\hat{f}_{d}(0)}{T}}},
\end{aligned}
\end{equation}
 where we set $h=4$   as the time horizon in the experiments, and $\bar{d}=\sum_{t=1}^{T}d_{t}=\sum_{t=1}^{T}\Big(\text{RMSE}_{t}^{2}(1)-\text{RMSE}_{t}^{2}(2)\Big)$ is the average of the distance between the  instantaneous   RMSE produced by two different methods (method $1$ and method $2$) at time index $t$. Specifically, we have $d_{t}=\text{RMSE}_{t}^{2}(1)-\text{RMSE}_{t}^{2}(2)$ and
 \begin{equation}
 \begin{aligned}
 \text{RMSE}_t(1)&=\sqrt{\frac{\sum^{M_t}_{j=1}\|\tilde{\bm{y}}^j_t(1)-\bm{y}^j_t\|_2^2}{\sum^{M_t}_{j=1}\|\bm{y}^j_t\|_2^2}},\\
 & \text{with}\quad \tilde{\bm{y}}^j_t(1) = \phi (\bm{D}^{est}_t(1),\tilde{\bm{x}}^j_t(1)), \|\tilde{\bm{x}}^j_t(1)\|_0=\eta_0,
 \end{aligned}
 \end{equation}
 where $\tilde{\bm{y}}^j_t(1)$ is the estimated readings calculated from method $1$ for the $j$-th reading at time $t$.  Furthermore, in (\ref{eq:DM}), we have \begin{equation}\small
\begin{aligned}
\hat{f_{d}}(0)=&\sum_{k=-(T-1)}^{T-1}\Big(\mathbb{I}(\frac{k}{h-1})\frac{1}{T}\\
&\times\sum_{t=|k|+1}^{T}(d_{t}-\bar{d})(d_{t-|k|}-\bar{d})\Big),
\end{aligned}
\end{equation} where the indicator function is expressed as \begin{equation}
\mathbb{I}(\frac{k}{h-1})=\begin{cases}
\begin{array}{c}
1\\
0
\end{array} & \begin{array}{c}
|\frac{k}{h-1}|\leq1,\\
otherwise.
\end{array}\end{cases}
\end{equation}

  To interpret the HLNDM statistic, we define the null hypothesis $H_{0}:\mathbb{E}(d_{t})=0,\forall t$, which means that the two methods have the same compression accuracy in terms of  statistics. The null hypothesis of no difference will be rejected if the computed HLNDM statistic falls outside the range of $-z_{\alpha/2}$ to  $z_{\alpha/2}$, i.e., 
\begin{equation}
|HLNDM|>z_{\alpha/2},
\end{equation}where $z_{\alpha/2}$ is the upper (or positive) $z$-value from the standard normal table corresponding to half of the desired $\alpha$ significance level of the test.  In the experiment, we set the proposed method as the reference method $1$ to calculate the HLNDM values for the baselines.
 \subsubsection{Experiment Procedures}\label{sec:proce}
 In the experiments, both online and offline methods are considered. We first introduce the experiment procedures for the {\em offline methods (Baseline 4 and Baseline 5)}.
 \begin{itemize}
 	\item {\em Offline Training}: The first $4000$ temperature readings are used to train the weights and dictionaries in Baseline 4 and Baseline 5. For Baseline 4 \cite{spatial},  the first three terms in the training loss  \cite[Eq.(4)]{spatial} are considered, and  a spatiotemporal long short-term memory (ST-LSTM) is trained as the deep encoder $f_{enc}(\cdot)$ and a $4$-layer rectified linear unit (ReLU) neural network is trained as the node decoder $f_n(\cdot)$ with the Adam optimizer using $200$ epochs. The edge decoder $f_e(\cdot)$ follows the expression in \cite[Eq.(10)]{spatial}. The hyper-parameters are fine-tuned to be $m=6$,  $K=50$, $d_h=20$, $\lambda_e=0.25$ and $\lambda_n=0.25$ for better performance. In addition, the choices of $\lambda$ are listed in Table  \ref{tab:lp}. For Baseline 5 \cite{akhtar2017nonparametric}, the upper bound of the number of atoms for the dictionary is set to  $K=80$ and the Beta distribution parameters are set to $a_0=b_0 = \frac{4000}{8}$ for better performance. Other parameters have the same values as those in \cite[Section V]{akhtar2017nonparametric}.
 	\item {\em Compression}: The remaining $4593$ readings are grouped into $766$ mini-batches with mini-batch size $M_t=6$.\footnote{The last mini-batch of readings has the mini-batch size $M_{766}=3$.}   Then each mini-batch of readings is provided to the offline dictionary learning methods at one specific time index to test the performance of the offline learned dictionary.  Specifically,  for Baseline 4, the sparse code for the $j$-th test reading at time $t$, $\bm{y}^j_t\quad (j\in[M_t])$, is calculated by  solving the following sparse coding problem provided in \cite{spatial}:
 		 \begin{equation}
 		\tilde{\bm{x}}^j_t = \mathcal{T}_{\eta_0}(\arg\min\|f_{enc}(\mathcal{G}(\bm{y}^j_t))-\bm{D}^{est}_t\bm{x}\|_F^2+\lambda\|\bm{x}\|_1),
 		\end{equation} with the sklearn Lasso toolbox.\footnote{Implementation details can be found at \url{http://scikit-learn.org/stable/modules/generated/sklearn.linear_model.Lasso.html}} The estimated reading $\tilde{\bm{y}}^j_t$ is calculated by 
 		$\tilde{\bm{y}}^j_t = \phi(\bm{D}^{est}_t,\tilde{\bm{x}}^j_t)=f_n(\bm{D}^{est}_t\tilde{\bm{x}}^j_t),$
 		where $\bm{D}^{est}_t$ is equal to the offline learned dictionary for all $t$, $f_{enc}(\cdot)$ and $f_n(\cdot)$ are the offline learned encoder and decoder, and $\mathcal{G}(\bm{y}^j_t)$ is the graph embedding for  $\bm{y}^j_t$ given by \cite[Section II-B]{spatial}. $ \mathcal{T}_{\eta_0}(\bm{a})$ selects $\eta_0$ elements in $\bm{a}$ with the largest $\ell_2$-norm and sets the other elements to zero. For Baseline 5,  the sparse code for the $j$-th test reading  at time $t$, $\bm{y}^j_t\quad (j\in[M_t])$, is calculated by solving the following problem provided by \cite{akhtar2017nonparametric}: \begin{equation}
 		\tilde{\bm{x}}^j_t = \underset{\|\bm{x}\|_0\leq \eta_0}{\arg\min}\|\bm{y}^j_t-\bm{D}^{est}_t\bm{x}\|_2^2,
 		\end{equation} 	with the sklearn OMP toolbox\footnote{Implementation details can be found at \url{https://scikit-learn.org/stable/modules/generated/sklearn.linear_model.OrthogonalMatchingPursuit.html}}, and the estimated reading $\tilde{\bm{y}}^j_t$ is calculated by $
 		\tilde{\bm{y}}^j_t = \phi(\bm{D}^{est}_t,\tilde{\bm{x}}^j_t)=\bm{D}^{est}_t\tilde{\bm{x}}^j_t$,
 		where $\bm{D}^{est}_t$ is equal to the offline learned dictionary for all $t$.
 
	\item {\em Performance Metric Calculation}: 	The RMSE values and the HLNDM test results under different compression ratios  are calculated  for Baseline 4 and Baseline 5 according to (\ref{eq:RSME}) and (\ref{eq:DM}).
\end{itemize}
Next, we elaborate on the experiment procedures for the {\em online methods (the proposed method, Baseline 1, Baseline 2, Baseline 3)}.
	\begin{itemize}
		\item {\em Initialization}: Since there is no offline training in the online methods, the last $4593$ readings are provided to the online methods.   The first $100$  readings  are utilized to initialize the dictionary $\bm{D}_0$ by the batch version of each online method with $20$ iterations.
		\item {\em Online Learning}: The remaining $4493$ readings are grouped into $749$ mini-batches with mini-batch size $M_t=6$.\footnote{The last mini-batch of readings has the mini-batch size $M_{749}=5$.} The readings $\bm{y}^j_t\quad (j\in[M_t])$  at time $t$ are provided  to the online methods to produce $\bm{D}_t$, sequentially along $t=1,\ldots,749$.
		\item {\em Online Compression}: Each reading in the mini-batch is compressed  by the sparse code  after obtaining the dictionary $\bm{D}_t$.  For the proposed method,  Baseline 1, and Baseline 2, the sparse code for the $j$-th reading at time $t$, $\bm{y}^j_t\quad (j\in[M_t])$, is calculated using
			\begin{equation}
			\tilde{\bm{x}}^j_t = \mathcal{T}_{\eta_0}(({\bm{D}_t})^{\text T}\bm{y}^j_t).
			\end{equation}
		 For Baseline 3,  the sparse code is
			\begin{equation}
		\tilde{\bm{x}}^j_t = \mathcal{T}_{\eta_0}(\arg\min\|\bm{y}^j_t-\bm{D}_t\bm{x}\|_F^2+\lambda\|\bm{x}\|_1).
			\end{equation}
				For all online methods, the estimated reading  $\tilde{\bm{y}}^j_t$ is calculated by $
				\tilde{\bm{y}}^j_t = \phi(\bm{D}^{est}_t,\tilde{\bm{x}}^j_t)=\bm{D}_{t}\tilde{\bm{x}}^j_t$.
		\item {\em Performance Metric Calculation}: Then, the RMSE values and the HLNDM test results under different compression ratios  are calculated  for the proposed method,  Baseline 1, Baseline 2, and Baseline 3 according to (\ref{eq:RSME}) and (\ref{eq:DM}).
\end{itemize}

\subsubsection{Performance Comparison and Discussion}
 The RMSE, HLNDM and CPU time  comparison among different dictionary learning schemes under different compression ratios are given in Table \ref{tab:lp}. The results demonstrate that the proposed online ODL scheme achieves a lower RMSE  than the other methods for  the data compression task under different compression ratios. The proposed online ODL scheme keeps tracking the input readings and  adapts the dictionary learning within a small optimization space. Hence it is capable of finding a good dictionary with fewer streaming data received at each time index.

 To interpret the HLNDM results,  we adopt a widely used significance level of $\alpha=0.05$ with $z_{\alpha/2}=1.96$. The HLNDM statistic results and the RMSE results show that the proposed method has better sensor data compression performance  in most cases. Though the HLNDM between the proposed method and Baseline 2  is $1.51$ ($<1.96$) when $\eta_0=17$,  the probability that these two methods have the same level of accuracy is only $6.55\%$.  
 
  We count the per-time-slot CPU time to show the computational efficiency of the proposed method. For the offline methods, we  count the per-time-slot CPU time  on the test reading compression stage; i.e., we ignore the offline training cost. In this case, the proposed online ODL method still costs fewer computational resources. This is because  it only executes one iteration at each time index with a simple thresholding of  a matrix-vector product to reconstruct the sensor readings, while the offline methods  require solving an optimization problem.
 
 Fig. \ref{fig:forcast} shows the compression results  of the proposed online ODL method for $1000$   temperature readings of the $45$-th wireless sensor  from 21:00, Oct. 27, 2017 to 12:00, Dec. 08, 2017 under different compression ratios. As shown in this figure, the proposed method can compress the temperature readings with a satisfactory accuracy. The maximum compression errors  are $0.92$ $^{\circ}C$ with $\eta_0=8$ and $0.23$ $^{\circ}C$ with $\eta_0=17$.
 \begin{table*}[htpb]
 	\footnotesize
 	\centering
 	\caption{ {  The performance of different methods in compressing  sensor readings of  temperature in 2017 in Krakow, Poland \cite{poland}}}
 	\begin{threeparttable}
 		{  \begin{tabular}{lccccccccc}
 				\toprule[2pt]
 				\multicolumn{1}{c}{	Compression ratio} &\multicolumn{3}{c}{$28$  $(\eta_0=2)$}  &  \multicolumn{3}{c}{$7$ $(\eta_0=8)$}&  \multicolumn{3}{c}{$5$ $(\eta_0=10)$}  \\ \hline
 				Methods \tnote{*} &  Time (ms)      &    RMSE  &HLNDM   & Time (ms)   &   RMSE    &HLNDM      & Time (ms)         &     RMSE &HLNDM    \\ 
 				\hline
 				Proposed&     \bf{1.018}      &    \bf{4.82\%}   & Ref   &    \bf{1.021}      &   \bf{  2.74\%}   & Ref     &   \bf{1.032}     &   \bf{2.53\%} & Ref    \\
 				\hline
 				Baseline 1 (SFW)&   1.022      &      22.53\% &  7.56  &     1.030   &     14.78\% & 3.57   &   1.034    &   13.90\%  &   3.48 \\ 
 				\hline
 				Baseline 2 ($\ell_4$\_NoncvxSFW)&    1.076  &      5.22\%  &  6.212  &     1.083    &      3.16\%  &2.65   &    1.088     &   3.04\%&2.70    \\
 				\hline
 				Baseline 3 (Online AODL) &  8.932      &      47.45\% &  4.38  &     19.58     &     25.06\% &   3.29 &    39.65   &    24.71\% & 2.11  \\
 				\hline
 				Baseline 4 (Offline-DeepSTGDL) &     4.606      &      24.8\%  & 2.42 &      10.14    &      10.66\% &  5.47 &    18.04      &    7.27\% &3.41   \\
 				\hline
 				Baseline 5 (Offline-NCBDL) &    2.717     &      34.72\%  & 5.94 &       4.028   &      24.87\% &   6.17 &   4.347    &    24.72\%  & 6.51 \\
 				\hline
 				\hline
 				\multicolumn{1}{c}{	Compression ratio} &\multicolumn{3}{c}{$3$  $(\eta_0=17)$}  &  \multicolumn{3}{c}{$2$ $(\eta_0=25)$}&  \multicolumn{3}{c}{$1$ $(\eta_0=35)$}  \\ \hline
 				Methods \tnote{*} &     Time (ms)      &    RMSE  &HLNDM   & Time (ms)    &   RMSE    &HLNDM      &  Time (ms)          &     RMSE &HLNDM    \\ 
 				\hline
 				Proposed&     \bf{1.036}      &     \bf{1.97\%}   & Ref   &    \bf{1.040}      &   \bf{  1.20\%}   & Ref     &   \bf{1.104}     &   \bf{0.68\%} & Ref    \\
 				\hline
 				Baseline 1 (SFW)&     1.040    &      9.68\% & 3.04   &     1.063   &     9.40\% & 2.37   &   1.163   &   8.81\%  & 7.54   \\ 
 				\hline
 				Baseline 2 ($\ell_4$\_NoncvxSFW)&    1.089   &      2.40\%  &  1.51 &     1.096   &      1.48\%  &2.91   &    1.106    &   1.08\%&  2.17  \\
 				\hline
 				Baseline 3 (Online AODL) &      105.9    &      19.54\% & 4.14   &      111.2      &     19.49\% & 5.34   &    137.7    &    17.90\% &  6.15 \\
 				\hline
 				Baseline 4 (Offline-DeepSTGDL) &    38.69   &      7.00\%  &5.28  &    39.29   &      6.74\% &  5.21 &    53.08     &    6.15\% &  6.28 \\
 				\hline
 				Baseline 5 (Offline-NCBDL) &    6.057      &      16.86\%  & 6.80 &       7.550  &      13.51\% & 5.33   &   9.675     &    13.42\%  &6.24  \\
 				\bottomrule[2pt]
 		\end{tabular}}	\label{tab:lp}
 		\begin{tablenotes}{  
 				\footnotesize
 				\item[*] For $\eta_0 = 2,8,10,17,25,35$, the hyper-parameter $\lambda$ for Baseline 3 is set to $18,5,1,0.3,0.2,0.1$, and $\lambda$ for Baseline 4 is set to $20,8,4,0.4,0.4,0.2$, for better performance.}
 			
 		\end{tablenotes}
 	\end{threeparttable}

 \end{table*}
\begin{figure}[htpb]
	\centering
	\includegraphics[width=0.9\linewidth]{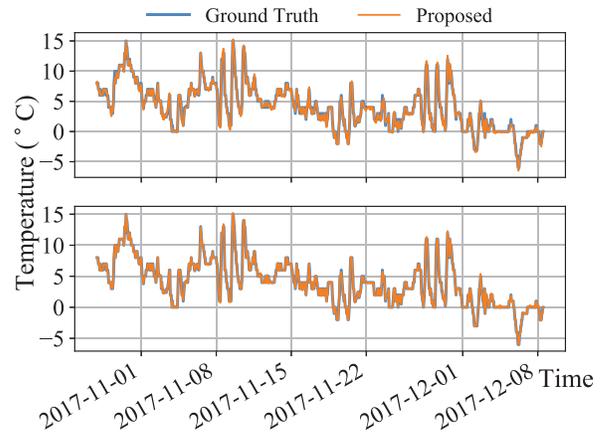}
	\caption{{ Sensor temperature reading compression results of the proposed method for the $45$-th wireless sensor from 21:00, Oct. 27, 2017 to 12:00, Dec. 08, 2017. Top to bottom rows: $\eta_0=8,17$, compression ratio $=7,3$.} }
	\label{fig:forcast}
\end{figure}

 \subsection{Generalization Study}
 To demonstrate the generalization of the proposed Algorithm \ref{alg:genFW} based on the generic problem (\ref{eq:genNoncvx}),
we evaluate the performance of  Algorithm \ref{alg:genFW} on the Sparse Principal Component Analysis (SPCA) problem \cite{shen2008sparse,journee2010generalized} which is a powerful framework to find a sparse principal component for seeking a reasonable trade-off between the statistical fidelity and interpretability of data. Specifically, for the single-unit SPCA problem, one can  have the following formulation:
\begin{equation}
\begin{aligned}\label{eq:spca}
\underset{\bm{\bm{z}}\in\mathbb{B}^{N}}{min}\mathbb{E}_{\bm{y}\sim P}[-\bm{\bm{z}}^{T}\bm{y}\bm{y}^{T}\bm{\bm{z}}+\lambda H_{\mu}(\bm{\bm{z}})],
\end{aligned}
\end{equation}
where ${\ensuremath{\mathbb{B}^{N}}}$ is the $N$ dimensional closed unit ball, which is a convex set, $H_{\mu}(\cdot)$ is the Huber loss with parameter $\mu$ to promote the sparsity of the principal component, and $\lambda$ is a regularization parameter. The objective of Problem (\ref{eq:spca})  is non-convex. Problem (\ref{eq:spca}) can be solved by the proposed generic Algorithm \ref{alg:genFW}, with the LMO over the unit ball as $\bm{s}_{t}=\arg\min_{\bm{s}\in\mathbb{B}^{N}}\langle\bm{g}_{t},\bm{s}\rangle=\frac{-\bm{g}_{t}}{\|\bm{g}_{t}\|_{2}}$. 
 
 We test the convergence property of the proposed Algorithm \ref{alg:genFW} for Problem (\ref{eq:spca}) in a number  of experimental settings with synthetic data. In the experiments, we follow the procedures in \cite{shen2008sparse,journee2010generalized} to generate random data with a covariance matrix $\boldsymbol{\text{\ensuremath{\Sigma}}}=\bm{V\Upsilon V}^{T}\in\mathbb{R}^{N\times N}$ containing a sparse leading eigenvector. To do so, $\boldsymbol{\text{\ensuremath{\Sigma}}}$ is synthesized by constructing $\bm{\Upsilon}$ and $\bm{V}$. Specifically, we have $\bm{\Upsilon}=diag(\bm{\gamma})$, where $\bm{\gamma}=[100,1,1,\ldots,1]$.  To synthesize $\bm{V}=[\bm{v}_{1},\bm{v}_{2},\ldots,\bm{v}_{N}]$, we first construct the sparse leading eigenvector $\bm{v}_{1}$. Specifically, the $i$-th component of $\bm{v}_{1}$ is constructed by the following expression:
 \begin{equation}
 v_{1,i}=\begin{cases}
 \begin{array}{c}
 \frac{1}{\text{\ensuremath{\sqrt{q}}}}\\
 0
 \end{array} & \begin{array}{c}
 i=1,\ldots,q,\\
 otherwise.
 \end{array}\end{cases}
 \end{equation}
 Next, vectors $\bm{v}_{2},\ldots,\bm{v}_{N}$ are randomly drawn from $\mathbb{S}^{N-1}$ until a full-rank $\bm{V}$  is obtained.  Then, we apply the Gram-Schmidt orthogonalization method to the full-rank $\bm{V}$ to obtain an orthogonal matrix $\bm{V}$. Finally, $T$  mini-batches of data with mini-batch size $B$ are generated by drawing independent samples from a zero-mean normal distribution with covariance matrix $\boldsymbol{\Sigma}=\bm{V\Upsilon V}^{T}$. The resultant data samples are denoted by $\bm{y}_{t}^{j}\quad(j\in[B],1\leq t\leq T)$, where $t$ is the index for the mini-batch and $j$ is the index for a sample within one mini-batch.
 
 The target for the algorithm is to find the sparse leading principal component $\bm{v}_{1}$ with the streaming input $\bm{y}_{t}^{j}\quad(j\in[B],1\leq t\leq T)$ by solving Problem (\ref{eq:spca}). We set $\mu=0.2$, $\lambda=1$ for the problem and generate $T=3\times 10^3$ mini-batches of data from the above procedures. In the experiments, we test the recovery error by varying the dimensions $N$, the mini-batch size $B$, and the number of non-zero elements in the principal component $q$. The recovery error is defined by \begin{equation}
\text{Error}_{t}=|1-|\bm{z}_{t}^{T}\bm{v}_{1}||
 \end{equation} since the maximum of $|\bm{z}_{t}^{T}\bm{v}_{1}|$ is $1$, and it is achieved when $\bm{z}_{t}=\bm{v}_{1}$. The results with $\text{Error}_{t}$ against the number of iterations are shown in Fig. \ref{fig:spca}, where $\text{Error}_{t}$ at each $t$ is averaged over $100$ independent Monte Carlo trials. As indicated by the results under various experimental settings, the proposed method for generic problem (\ref{eq:genNoncvx}) generalizes well to this SPCA problem.
 \begin{figure}[htbp]
 	\centering
 	\includegraphics[width=0.85\linewidth]{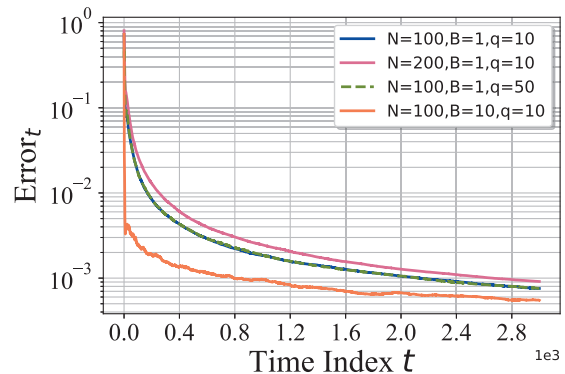}
 	\caption{{ Performance of the proposed Algorithm 1  in solving the SPCA problem.}}
 	\label{fig:spca}
 \end{figure}}

\section{Conclusion}\label{sec:concl}
We have proposed in this paper a novel online orthogonal dictionary learning scheme with a new relaxed problem formulation, a Frank-Wolfe-based online algorithm, and the convergence analysis. Experiments on synthetic data and a real-world data set  show the superior performance of our proposed scheme compared to the baselines and verify the correctness of our theoretical results. {Future focus on  application will include a distributed or federated version of the proposed scheme to further leverage the spatially distributed data. As for the theory, the sample complexity analysis of the proposed scheme will be an important direction since it characterizes how close the result returned by the proposed scheme at a given iteration is to the ground truth dictionary.}


%

\appendices
\section{Auxiliary Lemmas} 
\begin{lemma}\label{lem:Young's}
	Let $\langle \bm{A},\bm{B} \rangle$ be the Frobenius inner product of real matrices $\bm{A}\in \mathbb{R}^{N\times M}$ and $\bm{B}\in \mathbb{R}^{N \times M}$, then we have for any $\epsilon>0$
	\begin{equation}
	\langle \bm{A},\bm{B} \rangle \leq \frac{\epsilon^p}{p}\|vec(\bm{A})\|_p^p+\frac{\epsilon^{-q}}{q}\|vec(\bm{B})\|_q^q.
	\end{equation}
\end{lemma} 
\begin{proof}
	According to the definition of the Frobenius norm, we have
	\begin{equation}
	\begin{aligned}
	\langle \bm{A},\bm{B} \rangle =&vec(\bm{A})^Tvec(\bm{B})=\sum_{i=1}^M\sum_{j=1}^NA_{j,i}B_{j,i}\\
	=&\sum_{i=1}^M\sum_{j=1}^N\epsilon  A_{j,i}\epsilon^{-1}B_{j,i}\leq\sum_{i=1}^M\sum_{j=1}^N( \frac{\epsilon^p}{p}A_{j,i}^p+\frac{\epsilon^{-q}}{q}B_{j,i}^q),
	\end{aligned}
	\end{equation}
	where the last inequality is according to   Young's Inequality \cite{young1912classes}.
\end{proof}
\begin{lemma}\label{lem:finitesum}
	Let $f(s)$ be  any decreasing function of $s$. We have
	\begin{equation}
	\int_{\mu=a}^{b+1}f(\mu)d\mu\leq\sum_{s=a}^{b}f(s)\leq\int_{\mu=a-1}^{b}f(\mu)d\mu.
	\end{equation}
\end{lemma} 

\section{Proof Theorem \ref{thm:relax}}\label{proof:relax}
Let $\bm{W}=\bm{D}^T\bm{D}^{true}\in\mathbb{R}^{N\times N}$,
we  have

\begin{equation}\label{eq:maxF}
\begin{aligned}
&\mathbb{E}_{\bm{y}\sim \mathcal{BG}(\theta)}[-\|\bm{D}^{\text T}\bm{y}\|^3_3] =-\sum_{n=1}^{n=N}\mathbb{E}[|\bm{W}_{n,:}\bm{x}|^{3}]\\
=&-\sum_{n=1}^{n=N}\mathbb{E}[|(\bm{W}_{n,:}\odot\bm{b}^T)\bm{g}|^{3}],\\
\end{aligned}
\end{equation}
where we denote $\bm{b}\sim_{i.i.d}\mathbb{B}_{sp}(\theta)$ and
$\bm{g}\sim_{i.i.d}\mathcal{CN}(0,1)$. Using the rotation-invariant
property of Guassian random variables, we have $-\mathbb{E}[|\bm{W}_{n,:}\odot\bm{b}^T\bm{g}|^{3}]=-\gamma_{1}\mathbb{E}[||\bm{W}_{n,:}\odot\bm{b}^T||_{2}^{3}]$
with $\gamma_{1}=\frac{2^{3/2}}{\sqrt{\pi}}$ calculated
by the 3rd-order non-central moment of the Gaussian distribution.

For Problem (\ref{eq:DLonlinel31}), we have $\bm{D}\in\mathbb{O}(N,\mathbb{R})$ and $||\bm{W}_{n,:}||_{2}=||\bm{D}_{:,n}^{T}\bm{D}^{true}||_{2}\leq1$.
Hence, we have $0\leq\mathbb{E}[||\bm{W}_{n,:}\odot\bm{b}^T||_{2}^{3}]\leq\mathbb{E}[||\bm{W}_{n,:}\odot\bm{b}^T||_{2}^{2}]=\theta$.
The equality holds if and only if $||\bm{W}_{n,:}\odot\bm{b}^T||_{2}\in\{0,1\}$
for all $\bm{b}$, which is only satisfied at $\bm{W}_{n,:}\in\{\pm \bm{e}^T_{i}:i\in[N] \}$\cite{bai2018subgradient}.
Therefore, we have $\mathbb{E}_{\bm{y}\sim \mathcal{BG}(\theta)}[-\|\bm{D}^{\text T}\bm{y}\|^3_3]=-\sum_{n=1}^{n=N}\mathbb{E}[||\bm{W}_{n,:}\odot\bm{b}^T||_{2}^{3}]\geq- N \gamma_{1}\theta$ and the minimum is achieved when $\bm{W}_{n,:}\in\{\pm \bm{e}^T_{i}:i\in[N] \}$.
Furthermore, if $\bm{W}_{n_1,:}=\pm  \bm{e}^T_{i}$
and $\bm{W}_{n_2,:}=\pm \bm{e}^T_{i}$, then $Tr(\bm{W}_{n_1,:}\bm{W}^T_{n_2,:})=\pm 1$.
However, we have $Tr(\bm{W}_{n_1,:}\bm{W}^T_{n_2,:})=Tr((\bm{D}_{:,n_1}^{T}\bm{D}^{true})(\bm{D}_{:,n_2}^{T}\bm{D}^{true})^T)=Tr(\bm{D}_{:,n_1}^{T}\bm{D}_{:,n_2})=0$.
This indicates that two different columns of $\bm{W}$ cannot
simultaneously equal the same standard basis vector. Hence,  $\mathbb{E}_{\bm{y}\sim \mathcal{BG}(\theta)}[-\|\bm{D}^{\text T}\bm{y}\|^3_3] $
achieves the minimum $-N\gamma_{1}\theta$ when $\bm{D}=\bm{D}^{opt}$
with $(\bm{D}^{opt})^T\bm{D}^{true}=\boldsymbol{\Xi}$. In other words, the optimal solution of Problem (\ref{eq:DLonlinel31}) is $\bm{D}^{opt}=\bm{D}^{true}\boldsymbol{\Xi}^T$.

For Problem $\mathscr{P}$, we have $\bm{D}\in\mathbb{B}_{sp}(N,\mathbb{R})$, but  $||\bm{W}_{n,:}||_{2}=||\bm{D}_{:,n}^{T}\bm{D}^{true}||_{2}\leq1$ still holds. Hence,  we also have $\mathbb{E}_{\bm{y}\sim \mathcal{BG}(\theta)}[-\|\bm{D}^{\text T}\bm{y}\|^3_3]=-\sum_{n=1}^{n=N}\mathbb{E}[||\bm{W}_{n,:}\odot\bm{b}^T||_{2}^{3}]\geq- N \gamma_{1}\theta$ and the minimum is achieved when $\bm{W}_{n,:}\in\{\pm \bm{e}^T_{i}:i\in[N] \}$. Since $\bm{D}^{opt}=\bm{D}^{true}\boldsymbol{\Xi}^T\in\mathbb{B}_{sp}(N,\mathbb{R})$ is  feasible for  Problem $\mathscr{P}$,  $\bm{D}^{opt}$ is also an optimal solution for  Problem $\mathscr{P}$. This finishes the proof of  Theorem \ref{thm:relax}.

{Then we will demonstrate that if the minimizers of  Problem $\mathscr{P}$ are restricted to be full rank, then these minimizers are  also the minimizers of Problem  (\ref{eq:DLonlinel31}). The minimum of  Problem $\mathscr{P}$ is achieved  when $\bm{W}_{n,:}\in\{\pm \bm{e}^T_{i}:i\in[N] \}$. If there are two columns of $\bm{W}$ that can simultaneously equal the same standard basis vector, say $\bm{W}_{n_1,:}=\pm  \bm{e}^T_{i}$
	and $\bm{W}_{n_2,:}=\pm \bm{e}^T_{i}$, then $|Tr(\bm{W}_{n_1,:}\bm{W}^T_{n_2,:})|=1$. However, if we we have $\bm{D}\in \mathbb{B}_{sp}(N,\mathbb{R})$ and $\bm{D}$ being full rank,  then $|Tr(\bm{W}_{n_1,:}\bm{W}^T_{n_2,:})|=|Tr((\bm{D}_{:,n_1}^{T}\bm{D}^{true})(\bm{D}_{:,n_2}^{T}\bm{D}^{true})^T)|=|Tr(\bm{D}_{:,n_1}^{T}\bm{D}_{:,n_2})|<1$.}
\section{Proof of Lemma \ref{lem:FWgap}}\label{proof:gap}
The definition of the stationary points for a constraint problem is
\begin{defn}[Definition of Stationary Points]\label{def:sta}
	We will say that $\bm{X}^*\in \mathcal{C}$ is a stationary point if
	$$\langle\nabla F_{gen}(\bm{X}^*),\bm{S}-\bm{X}^*\rangle\geq 0 ,\forall \bm{S}\in\mathcal{C}. $$
\end{defn}
Then, we first show the sufficient condition.
If $g^{gen}_t = 0$, we have $\underset{\bm{S}\in\mathcal{C}}{\min}\langle\nabla F_{gen}(\bm{X}_{t-1}),\bm{S}-\bm{X}_{t-1}\rangle = 0.$ According to Definition \ref{def:sta},  obviously $\bm{X}_{t-1}$ is a stationary point.

Next, we show the necessary condition. If $\langle\nabla F_{gen}(\bm{X}_{t-1}),\bm{S}-\bm{X}_{t-1}\rangle\geq 0 ,\forall \bm{S}\in\mathcal{C}$, then we have
$\langle -\nabla F_{gen}(\bm{X}_{t-1}),\bm{S}-\bm{X}_{t-1}\rangle\leq 0 ,\forall \bm{S}\in\mathcal{C},$
which indicates
\begin{equation} \label{eq:proofgapgt}
\begin{aligned}
\underset{\bm{S}\in\mathcal{C}}{\max}\langle-\nabla F_{gen}(\bm{X}_{t-1}),\bm{S}-\bm{X}_{t-1}\rangle = g^{gen}_t\leq 0.
\end{aligned}
\end{equation}
Let $\bm{S}^*=\underset{\bm{S}\in\mathcal{C}}{\arg\max}\langle-\nabla F_{gen}(\bm{X}_{t-1}),\bm{S}\rangle$, then we have $g^{gen}_t = \langle-\nabla F_{gen}(\bm{X}_{t-1}),\bm{S}^*-\bm{X}_{t-1}\rangle\geq 0$. Combining the result in (\ref{eq:proofgapgt}), we have that if $\langle\nabla F_{gen}(\bm{X}_{t-1}),\bm{S}-\bm{X}_{t-1}\rangle\geq 0 ,\forall \bm{S}\in\mathcal{C}$, then $g^{gen}_t=0$.

The sufficient and necessary conditions complete the proof.
\section{ Proof of Theorem \ref{thm:convgen} }\label{proof:conv1}
For simplicity, we let $F(\cdot)$ represent $F_{gen}(\cdot)$. To prove Theorem \ref{thm:convgen}, we first prove the {\em Iterates Contraction} by the following Lemma.
\begin{lemma}[Iterates Contraction]\label{lem:uppergap}
	Under Assumptions \ref{assum1}, the Frank-Wolfe gap for Problem (\ref{eq:genNoncvx}) using Algorithm (\ref{alg:genFW}) satisfies
	{  
		\begin{equation}
		\begin{aligned}
		\gamma_tg^{gen}_t&\leq F(\bm{X}_t)-F(\bm{X}_{t+1})+\gamma_t\sqrt{\epsilon_t}diam(\mathcal{C})\\
		&+\frac{L}{2}\gamma^2_tNdiam(\mathcal{C})^2,
		\end{aligned}
		\end{equation}}
	where $N$ is the dimension of the subspace that $\mathcal{C}$ belongs to and 
	\begin{equation}
	\epsilon_t:=\|\bm{G}_t-\nabla F(\bm{X}_{t-1})\|^2_F
	\end{equation} is the gradient estimation error.
\end{lemma}
\begin{proof}
	We introduce the auxiliary variable
	\begin{equation}
	\bm{S}^*_t:=\arg\underset{\bm{S}\in\mathcal{C}}{\max}\langle -\nabla F(\bm{X}_{t-1}), \bm{S}\rangle
	\end{equation}
	for the proof. At the $t$-th iteration, we have
	\begin{equation}
	\begin{aligned}
	F(\bm{X}_{t})\leq&F(\bm{X}_{t-1}+\gamma_t(\bm{S}_t-\bm{X}_{t-1}))\\
	\overset{(a)}{\leq}&F(\bm{X}_{t-1})+\gamma_t\langle\nabla F(\bm{X}_{t-1}),\bm{S}_t-\bm{X}_{t-1}\rangle \\
	&+\frac{L\gamma^2_t}{2}\|\bm{S}_t-\bm{X}_{t-1}\|^2_F\\
	=&F(\bm{X}_{t-1})+\gamma_t\langle \bm{G}_t, \bm{S}_t-\bm{X}_{t-1}\rangle\\
	&+\gamma_t\langle\nabla F(\bm{X}_{t-1})-\bm{G}_t,\bm{S}_t-\bm{X}_{t-1}\rangle\\ &+\frac{L\gamma^2_t}{2}\|\bm{S}_t-\bm{X}_{t-1}\|^2_F\\
	\overset{(b)}{\leq}&F(\bm{X}_{t-1})+\gamma_t\langle \bm{G}_t, \bm{S}^*_t-\bm{X}_{t-1}\rangle+\frac{L\gamma^2_t}{2}\|\bm{S}_t-\bm{X}_{t-1}\|^2_F\\
	&+\gamma_t\langle\nabla F(\bm{X}_{t-1})
	-\bm{G}_t,\bm{S}_t-\bm{X}_{t-1}\rangle \\
	=&F(\bm{X}_{t-1})+\gamma_t\langle \nabla F(\bm{X}_{t-1})-\bm{G}_t, \bm{S}_t-\bm{S}^*_t\rangle\\
	&+\gamma_t\langle\nabla F(\bm{X}_{t-1}),\bm{S}^*_t-\bm{X}_{t-1}\rangle +\frac{L\gamma^2_t}{2}\|\bm{S}_t-\bm{X}_{t-1}\|^2_F\\
	\overset{(c)}{\leq}&F(\bm{X}_{t-1})+\gamma_t \|\nabla F(\bm{X}_{t-1})-\bm{G}_t\|_F \|\bm{S}_t-\bm{S}^*_t\|_F\\
	-&\gamma_tg^{gen}_t +\frac{L\gamma^2_t}{2}\|\bm{S}_t-\bm{X}_{t-1}\|^2_F\\
	=&F(\bm{X}_{t-1})+\gamma_t \sqrt{\epsilon_t} diam(\mathcal{C})-\gamma_tg_{t} +\frac{L\gamma^2_t}{2}diam(\mathcal{C})^2\\
	\Rightarrow& 	\gamma_tg_{t}\leq F(\bm{X}_{t-1})-F(\bm{X}_{t})+\gamma_t\sqrt{\epsilon_t}diam(\mathcal{C})\\&+\frac{L}{2}\gamma^2_tdiam(\mathcal{C})^2,
	\end{aligned}	
	\end{equation}
	where $(a)$ is from Assumption  \ref{assum1}, $(b)$ is from Algorithm (\ref{alg:genFW}), and $(c)$ is from the Cauchy–Schwarz inequality.
\end{proof}
The next key ingredient for the proof is the diminishing gradient estimation error as formally shown in the following Lemma.
\begin{lemma}[Diminishing   Gradient Estimation Error]\label{lem:deminisherror}
	Let the gradient estimation error at the $t$-th iteration in Algorithm (\ref{alg:genFW}) defined as
	\begin{equation}
	\epsilon_t:=\|\bm{G}_t-\nabla F(\bm{X}_{t-1})\|^2_F.
	\end{equation} Using the updating rule of Algorithm (\ref{alg:genFW}), we have 
	\begin{equation}
	\begin{aligned}
	\mathbb{E}\big[\epsilon_t\big]\leq \frac{\max\{C_0,C^t_1\}}{(t+2)^{1/2}},
	\end{aligned}
	\end{equation}
	where  $C_0 = \|\nabla F(\bm{X}_0)\|^2_F$, and $C^t_1 = {\frac{16V}{M_t}+2L^2  diam(\mathcal{C})^2}$. 
\end{lemma}
\begin{proof}
	To prove Lemma \ref{lem:deminisherror}  we have the following bound with $\mathbb{E}_t$   the conditional expectation w.r.t. the randomness sampled at the $t$-th iteration, conditioned on all randomness up to the $t$-th iteration.
	{  
		\begin{equation}\label{graderror1}
		\begin{aligned}\small
		&\mathbb{E}_t[\|\bm{G}_t-\nabla F(\bm{X}_{t-1})\|^2_F]\\
		=&\mathbb{E}_t[\|\rho_t(\nabla F(\bm{X}_{t-1})-\frac{1}{|M_t|}\sum_{i\in M_t}\nabla f(\bm{X}_{t-1},\bm{y}_i))\\
		&+(1-\rho_t)(\nabla F(\bm{X}_{t-1})-\bm{G}_{t-1})\|^2_F]\\
		=&\mathbb{E}_t[\|\rho_t(\nabla F(\bm{X}_{t-1})-\frac{1}{|M_t|}\sum_{i\in M_t}\nabla f(\bm{X}_{t-1},\bm{y}_i))\\
		&+(1-\rho_t)(\nabla F(\bm{X}_{t-1})-\nabla F(\bm{X}_{t-2}))\\
		&+(1-\rho_t)(\nabla F(\bm{X}_{t-2})-\bm{G}_{t-1})\|^2_F]\\
		\end{aligned}
		\end{equation}
		\begin{equation}\label{graderror2}
		\begin{aligned}\small
		&(\ref{graderror1})=\rho^2_t\mathbb{E}_t[\|\nabla F(\bm{X}_{t-1})-\frac{1}{|M_t|}\sum_{i\i nM_t}\nabla f(\bm{X}_{t-1},\bm{y}_i)]\|^2_F\\
		&+(1-\rho_t)^2\|\nabla F(\bm{X}_{t-1})-\nabla F(\bm{X}_{t-2})\|^2_F\\
		&+(1-\rho_t)^2\|\nabla F(\bm{X}_{t-2})-\bm{G}_{t-1}\|^2_F\\
		&+2\rho_t(1-\rho_t)\mathbb{E}_t[\langle\nabla F(\bm{X}_{t-1})\\
		&-\frac{\sum_{i\in M_t}\nabla f(\bm{X}_{t-1},\bm{y}_i),\nabla F(\bm{X}_{t-1})-\nabla F(\bm{X}_{t-2})\rangle}{|M_t|}]\\
		&+2\rho_t(1-\rho_t)\mathbb{E}_t[\langle\nabla F(\bm{X}_{t-1})\\
		&-\frac{\sum_{i\in M_t}\nabla f(\bm{X}_{t-1},\bm{y}_i),F(\bm{X}_{t-2})-\bm{G}_{t-1}\rangle}{|M_t|}]+2(1-\rho_t)^2\\
		\times&\langle \nabla F(\bm{X}_{t-1})-\nabla F(\bm{X}_{t-2}),\nabla F(\bm{X}_{t-2})-\bm{G}_{t-1}\rangle\\
		\end{aligned}
		\end{equation}
		\begin{equation}\label{graderror3}
		\begin{aligned}\small
		&(\ref{graderror2})	=\rho^2_t\mathbb{E}_t[\|\nabla F(\bm{X}_{t-1})-\frac{1}{|M_t|}\sum_{i\in M_t}\nabla f(\bm{X}_{t-1},\bm{y}_i)\|^2_F]\\
		&+(1-\rho_t)^2\|\nabla F(\bm{X}_{t-1})-\nabla F(\bm{X}_{t-2})\|^2_F\\
		+&(1-\rho_t)^2\|\nabla F(\bm{X}_{t-2})-\bm{G}_{t-1}\|^2_F\\	+&2(1-\rho_t)^2\langle \nabla F(\bm{X}_{t-1})-\nabla F(\bm{X}_{t-2}),\nabla F(\bm{X}_{t-2})-\bm{G}_{t-1}\rangle\\
		\overset{(a)}{\leq}&\rho_t^2\frac{V}{|M_t|}+(1-\rho_t)^2\|\nabla F(\bm{X}_{t-1})-\nabla F(\bm{X}_{t-2})\|^2_F\\
		+&(1-\rho_t)^2\|\nabla F(\bm{X}_{t-2})-\bm{G}_{t-1}\|^2_F\\	+&2(1-\rho_t)^2\langle \nabla F(\bm{X}_{t-1})-\nabla F(\bm{X}_{t-2}),\nabla F(\bm{X}_{t-2})-\bm{G}_{t-1}\rangle\\
		\overset{(b)}{\leq}&\rho_t^2\frac{V}{|M_t|}+(1-\rho_t)^2\|\nabla F(\bm{X}_{t-1})-\nabla F(\bm{X}_{t-2})\|^2_F\\
		+&(1-\rho_t)^2\|\nabla F(\bm{X}_{t-2})-\bm{G}_{t-1}\|^2_F\\	+&(1-\rho_t)^2\Big(\frac{2}{\rho_t} \|\nabla F(\bm{X}_{t-1})-\nabla F(\bm{X}_{t-2})\|_F\\
		&+\frac{\rho_t}{2}\|\nabla F(\bm{X}_{t-2})-\bm{G}_{t-1}\|_F\Big)\\
		\overset{(c)}{\leq}&\rho_t^2\frac{V}{|M_t|}+(1-\rho_t)^2(1+\frac{2}{\rho_t})\gamma_{t-1}^2L^2diam(\mathcal{C})^2
		\\
		&+(1-\rho_t)^2(1+\frac{\rho_t}{2})\|\nabla F(\bm{X}_{t-2})-\bm{G}_{t-1}\|^2_F\\
		\overset{(d)}{\leq}&\rho_t^2\frac{V}{|M_t|}+\frac{2}{\rho_t}\gamma_{t-1}^2L^2diam(\mathcal{C})^2
		\\
		&+(1-\frac{\rho_t}{2})\|\nabla F(\bm{X}_{t-2})-\bm{G}_{t-1}\|^2_F,
		\end{aligned}
		\end{equation}}
	where $(a)$ is obtained using   Assumption \ref{assum1}, $(b)$ is according  to  Lemma \ref{lem:Young's} with $\epsilon=\frac{\rho_t}{2}$ and $p=q=2$ 
	, and $(c)$ is obtained from 
	\begin{equation}
	\begin{aligned}
	&\|\nabla F(\bm{X}_t)-\nabla F(\bm{X}_{t-1})\|^2_F\leq L^2 \|\bm{X}_t-\bm{X}_{t-1}\|^2_F\\
	=&L^2 \gamma^2_t\|\bm{S}_{t}-\bm{X}_{t-1}\|^2_F\leq L^2 \gamma^2_tdiam(\mathcal{C})^2
	\end{aligned}
	\end{equation}
	via Assumption (\ref{assum1}) and Algorithm (\ref{alg:genFW}). In $(d)$, we use the inequality $(1-\rho_t)^2(1+\frac{2}{\rho_t})\leq (1-\rho_t)(1+\frac{2}{\rho_t})\leq\frac{2}{\rho_t}$ and $(1-\rho_t)^2(1+\frac{\rho_t}{2})\leq (1-\rho_t)(1+\frac{\rho_t}{2})\leq(1-\frac{2}{\rho_t}).$

	Therefore, we have 
	\begin{equation}
	\begin{aligned}
	\mathbb{E}_t\big[\epsilon_t\big]\leq& \big(1-\frac{\rho_t}{2})\|\bm{G}_{t-1}-\nabla F(\bm{X}_{t-2})\|^2_F\\
	&+\frac{2L^2diam({\mathcal{C}})^2\gamma^2_t}{\rho_t}+\rho^2_t\frac{V}{M_t},
	\end{aligned}
	\end{equation}
	
	and \begin{equation}
	\begin{aligned}
	&\mathbb{E}[\epsilon_t]=
	\mathbb{E}_{0,1,\ldots,t}[\epsilon_t]\\
	\leq& \big(1-\frac{4(t+1)^{-1/2}}{2})\mathbb{E}_{0,1,\ldots,t-1}[\epsilon_{t-1}]\\
	&+\big(2L^2diam({\mathcal{C}})^2+16\frac{V}{M_t}\big)(\eta_0+t)^{-1}.
	\end{aligned}
	\end{equation}
	Then, using the result in	 \cite[Lemma 17]{pmlr-v84-mokhtari18a}, we have finished the proof. 
\end{proof}

We are now able to  to prove Theorem \ref{thm:convgen}. Based on Lemma \ref{lem:uppergap}, Lemma \ref{lem:deminisherror}, and Jensen's inequality,  we have
\begin{equation}
\begin{aligned}
\mathbb{E}[\gamma_tg^{gen}_t]&\leq \mathbb{E}[F(\bm{X}_{t-1})-F(\bm{X}_{t})]+\gamma_t\sqrt{\mathbb{E}[\epsilon_t]}diam(\mathcal{C})\\
&+\frac{L}{2}\gamma^2_tdiam(\mathcal{C})^2.
\end{aligned}		
\end{equation}

Define $C_*=\max F(\bm{X})- \min F(\bm{X})$, we have
\begin{equation}
\begin{aligned}
\underset{1\leq s\le t}{\inf}\mathbb{E}[g_s]\sum^{t}_{s=1}\gamma_s\leq& C_*+\sum^{t}_{s=1} \big(\gamma_s\sqrt{\frac{\max\{C_0,C^t_1\}}{(s+\eta_0+1)^{1/2}}}diam(\mathcal{C})\\
&+\frac{L}{2}\gamma^2_sdiam(\mathcal{C})^2\big).
\end{aligned}
\end{equation}
Hence, we have
\begin{equation}
\begin{aligned}
\underset{1\leq s\le t}{\inf}\mathbb{E}[g_s]&\leq \Big (C_*+\sum^{t}_{s=1} \big(\gamma_s\sqrt{\frac{\max\{C_0,C^t_1\}}{(s+\eta_0+1)^{1/2}}}diam(\mathcal{C})\\
&+\frac{L}{2}\gamma^2_sdiam(\mathcal{C})^2\big)\Big)/(\sum^{t}_{s=1}\gamma_s)\\
\leq&\Big(C_*+\sum^{t}_{s=1} \big(2(s+2)^{-1}\sqrt{\max\{C_0,C^t_1\}}diam(\mathcal{C})\\
&+\frac{L}{2}4(s+2)^{-3/2}diam(\mathcal{C})^2\big)\Big)/\Big(\sum^{t}_{s=1}2(s+2)^{-3/4}\Big)\\
\leq& \Big(C_*+\sum^{t}_{s=1} \big((s+2)^{-1}(2\sqrt{\max\{C_0,C^t_1\}}diam(\mathcal{C})\\
&+2L diam(\mathcal{C})^2)\big)\Big)/\Big(\sum^{t}_{s=1}2(s+2)^{-3/4}\Big)
\end{aligned}
\end{equation}
Since both $(s+2)^{-1}$ and  $(s+2)^{-3/4}$ are decreasing in terms of $s$, from Lemma \ref{lem:finitesum}, we have
\begin{equation}
\sum_{s=1}^{t}(s+2)^{-1}\leq \ln(2+t)-\ln(2),
\end{equation} 
and
\begin{equation}
\sum_{s=1}^t(s+2)^{-3/4}\geq 4((t+3)^{1/4}-3^{1/4})\geq \frac{2}{5}(t+3)^{1/4},\quad \forall s>1.
\end{equation} 
Hence, we have
\begin{equation}
\begin{aligned}
\underset{1< s\le t}{\inf}\mathbb{E}[g_s]&\leq 
5\Big(C_*+10(\ln(2+t)-\ln 2)\\
&\times(\sqrt{\max\{C_0,C_1\}}diam(\mathcal{C})+L diam(\mathcal{C})^2)\Big)\\
&/\Big(4(t+3)^{1/4}\Big)\\
\leq& \Big(5C_*+10\ln(t+2)(\sqrt{\max\{C_0,C_1\}}diam(\mathcal{C})\\
&+L diam(\mathcal{C})^2)\Big)/\Big(4(t+3)^{1/4}\Big).
\end{aligned}
\end{equation}
This finishes the proof.
\section{Proof of Lemma \ref{lem:polar}}\label{proof:polar}
Since $Polar(\bm{D})\in  \mathbb{O}(N,\mathbb{R})$ for $\bm{D} \in  \mathbb{B}_{sp}(N,\mathbb{R})$, we know that $\mathbb{B}_{sp}(N,\mathbb{R})$ from Lemma \ref{lem:conv}.

Define $Polar(\bm{D})$ as $\bm{D}^{O}$. Let $\bm{W} = \bm{D}^T\bm{D}_{true}$ and $\bm{W}^O = (\bm{D}^O)T\bm{D}_{true}$. It is easy to know that $\bm{W}^O$ has orthonormal rows while rows of  $\bm{W}^O$ lie in the unit ball.  Using  similar calculations as in Eq.(\ref{eq:maxF}) and the discussion below, we have 
\begin{equation}
-\sum_{n=1}^N\mathbb{E}[||\bm{W}^O_{n,:}\odot\bm{b}^T||_{2}^{3}]\leq-\sum_{n=1}^N\mathbb{E}[||\bm{W}_{n,:}\odot\bm{b}^T||_{2}^{3}].
\end{equation}
Hence, we have $F(Polar(\bm{D}))\leq F(\bm{D})$ from Eq.(\ref{eq:maxF}). This completes the proof.

\section{Proof of Lemma \ref{lem:check}}\label{proof:ODLcond}
We first show condition (1) is satisfied. $\forall  \bm{D}_1,\bm{D}_2\in\mathbb{B}_{sp}(N,\mathbb{R})$, we have
\begin{equation}
\begin{aligned}
&\|\bm{D}_1-\bm{D}_2\|_F =\sqrt{trace\big((\bm{D}_1-\bm{D}_2)^T(\bm{D}_1-\bm{D}_2)\big)}\\
\leq&\sqrt{trace\big(\bm{D}^T_1\bm{D}_1+\bm{D}^T_2\bm{D}_2\big)}\leq\sqrt{N+N}=\sqrt{2N}.
\end{aligned}
\end{equation}

For condition (2), we define $\bm{d}=vec(\bm{D})$ and $f_v(\bm{d})=f_v(vec(\bm{D}))=\nabla F(\bm{D})=\nabla \mathbb{E}_{\bm{y}\sim \mathcal{P}}[-\|\bm{D}^{\text T}\bm{y}\|^3_3]$. Since we have $ \|\bm{D}_1-\bm{D}_2\|_2=\|\bm{d}_1-\bm{d}_2\|_2$,  the following holds
\begin{equation}
\begin{aligned}
&\|\nabla f_v(\bm{d}_1)-\nabla f_v(\bm{d}_2)\|_2\leq \|\mathrm{D}_{\bm{d}}[f_v(\bm{d}_1)] \|\|\bm{d}_1-\bm{d}_2\|_2\\
&\leq\|\mathrm{D}_{\bm{d}}[f_v(\bm{d}_1)] \|_F\|\bm{d}_1-\bm{d}_2\|_2\\
&=\sqrt{\sum_{i=1}^{N^2}\sum_{j=1}^{N^2}\mathrm{D}_{\bm{d}}[f_v(\bm{d}_1)]_{i,j}}\|\bm{d}_1-\bm{d}_2\|_2,\\
\end{aligned}
\end{equation}
where $\mathrm{D}_{\bm{d}}[f_v(\bm{d})]$ is the differential of $f_v(\bm{d})$ in terms of $\bm{d}$. To prove condition (2), we only need to show that
\begin{equation} \label{proof:lip1}
\sqrt{\sum_{i=1}^{N^2}\sum_{j=1}^{N^2}\big(\mathrm{D}_{\bm{d}}[f_v(\bm{d}_1)]_{i,j}}\big)^2\leq \sqrt{\frac{2}{\pi}}N^{3/2}(N+1)\theta.
\end{equation}
Letting $\bm{W} = \bm{D}_1^T\bm{D}_{true}$ and using the strategy in \cite[B.1]{bai2018subgradient}, we have the $j$-th element  of $f_v(\bm{d}_1)$ being
\begin{equation}
f_v(\bm{d}_1)_j =\mathbb{E}_{\Omega}\Big[ \bm{D}_{true([j]\_N,:)}\big(\|\bm{W}^{\Omega}_{(\lceil\frac{j}{N}\rceil,:)}\|\bm{W}^{\Omega}_{(\lceil\frac{j}{N}\rceil,:)}\big)^T\Big],
\end{equation}
where $\Omega$ is used to denote the generic support set of the Bernoulli random viable contained in $\bm{x}$ with $\bm{y}=\bm{D}_{true}\bm{x}$.
Then, we have the $i,j$-th element in $\mathrm{D}_{\bm{d}}[f_v(\bm{d}_1)]$ as
\begin{equation} \label{proof:lip2}
\big(\mathrm{D}_{\bm{d}}[f_v(\bm{d}_1)]_{i,j}\big)^2 
\begin{cases}
\leq \Big(\mathbb{E}_{\Omega}\big[N \|\bm{W}^{\Omega}_{(\lceil\frac{j}{N}\rceil,:)}\|+\|\bm{W}^{\Omega}_{(\lceil\frac{j}{N}\rceil,:)}\|\big]\Big)^2,\\
\text{if} \quad N(\lceil\frac{j}{N}\rceil-1)+1\leq i \leq N(\lceil\frac{j}{N}\rceil)\\
= 0,
\text{ortherwise}.\\
\end{cases}
\end{equation}
From \cite[B.1]{shen2020complete}, we know that
\begin{equation} \label{proof:lip3}
\mathbb{E}_{\Omega}\Big[\|\bm{W}^{\Omega}_{(\lceil\frac{j}{N}\rceil,:)}\|\big]\leq \sqrt{\frac{2}{\pi}}\theta.
\end{equation}
Combining (\ref{proof:lip3}) with (\ref{proof:lip2}), we have  (\ref{proof:lip1}). This finishes the proof.

Condition (3) holds obviously due to the i.i.d assumption of $\bm{x}_t$.

To prove that condition (4) is satisfied, we have
\begin{equation}\label{proof:var}
\begin{aligned}
&\mathbb{E}\big[\|\frac{1}{M_t}\sum_{j\in[M_t]}-\nabla \|\bm{D}^{\text T}\bm{y}^j_t\|^3_3-\nabla F(\bm{D})\|_F^2\big]\\
\leq&\mathbb{E}\big[\|\frac{1}{M_t}\sum_{j\in[M_t]}-\nabla \|\bm{D}^{\text T}\bm{y}^j_t\|^3_3\|_F^2\big] \\
=& \frac{1}{M^2_t}\sum_{j\in[M_t]}\mathbb{E}\big[trace\big((\nabla \|\bm{D}^{\text T}\bm{y}^j_t\|^3_3)^T (\nabla \|\bm{D}^{\text T}\bm{y}^j_t\|^3_3)\big)\big]\\
\overset{(a)}{\leq}&\frac{1}{M_t}\mathbb{E}\big[\sigma^2\big(\nabla \|\bm{D}^{\text T}\bm{y}\|^3_3\big)\big]\\
=&\frac{1}{M_t}\mathbb{E}\big[\|(\bm{x}^T\bm{D}^T_{true}\bm{D}\odot|\bm{x}^T\bm{D}^T_{true}\bm{D}|)\bm{D}_{true}\bm{x}\|^2\big]\\
\leq&\frac{1}{M_t}\mathbb{E}\big[\||\bm{x}^T\bm{D}^T_{true}\bm{D}|^{\odot 2}\bm{D}_{true}\bm{x}\|^2\big]\\
\overset{(b)}{\leq}&\frac{1}{M_t}\sqrt{\big(\mathbb{E}\big[\||\bm{x}^T\bm{D}^T_{true}\bm{D}|^{\odot 2}\|^4\big]\big)}\sqrt{\big(\mathbb{E}\big[\|\bm{D}_{true}\bm{x}\|^4\big]\big)},\\
\end{aligned}
\end{equation}
{  where (a) is due to that} matrix $\nabla \|\bm{D}^{\text T}\bm{y}^j_t\|^3_3$ is rank one, (b) is due to the Cauchy-Schwarz inequality.
Let $\bm{D}^T_{true}\bm{D}=\bm{W}$, and $\bm{W}_{:,i}$ represent the $i$-th  column vector in $\bm{W}$. We have
\begin{enumerate}
	\item 
	{  
		\begin{equation}\label{proof:var1}
		\begin{aligned}
		&\mathbb{E}\big[\||\bm{x}^T\bm{D}^T_{true}\bm{D}|^{\odot 2}\|^4\big]=\mathbb{E}\big[(\sum_{i=1}^N\|\bm{x}^T\bm{W}_{:,i}\|^2)^2\big]\\
		=&\mathbb{E}\big[\sum_{i=1}^N\|\bm{x}^T\bm{W}_{:,i}\|^4\big]+\mathbb{E}\big[\sum_{i^{\prime}\neq i}\|\bm{x}^T\bm{W}_{:,i^{\prime}}\|^2\|\bm{x}^T\bm{W}_{:,i}\|^2\big]\\
		\leq&\sum_{i=1}^N\mathbb{E}\big[\|\bm{x}^T\bm{W}_{:,i}\|^4\big]\\
		&+\sum_{i^{\prime}\neq i}\big(\mathbb{E}\big[\|\bm{x}^T\bm{W}_{:,i^{\prime}}\|^4\big]\big)^{\frac{1}{2}}\big(\mathbb{E}\big[\|\bm{x}^T\bm{W}_{:,i}\|^4\big]\big)^{\frac{1}{2}}\\
		\leq&\sum_{i=1}^N 3\theta+\sum_{i^{\prime}\neq i} 3\theta = N^2 3\theta,
		\end{aligned}
		\end{equation}}
	where the last inequality is according to \cite[Lemma B.1]{shen2020complete}. 
	\item  Let 	 $\bm{D}_{true(i,:)}$ represent the $i$-th  row vector in $\bm{D}_{true}$.
	\begin{equation}\label{proof:var2}
	\begin{aligned}
	&\mathbb{E}\big[\|\bm{D}_{true}\bm{x}\|^4\big]=\mathbb{E}\big[(\sum_{i=1}^N\|\bm{D}_{true(i,:)}\bm{x}\|^2)^2\big]\\\leq&\sum_{i=1}^N 3\theta+\sum_{i^{\prime}\neq i} 3\theta = N^2 3\theta.
	\end{aligned}
	\end{equation}
\end{enumerate}
Substituting (\ref{proof:var1}) and (\ref{proof:var2}) into (\ref{proof:var})  we finish the proof.

{\footnotesize
	\bibliography{IEEEabrv,Reference}}

\end{document}